\documentclass{article}

\PassOptionsToPackage{numbers, compress}{natbib}
 \usepackage[preprint]{neurips_2026}

\usepackage[utf8]{inputenc} % allow utf-8 input
\usepackage[T1]{fontenc}    % use 8-bit T1 fonts
\usepackage{url}            % simple URL typesetting
\usepackage{booktabs}       % professional-quality tables
\usepackage{amsfonts}       % blackboard math symbols
\usepackage{nicefrac}       % compact symbols for 1/2, etc.
\usepackage{microtype}      % microtypography
\usepackage{xcolor}         % colors
\usepackage{graphicx}       % images

\usepackage[colorlinks,
            linkcolor=blue,
            anchorcolor=blue,
            citecolor=cyan]{hyperref}
\usepackage{hyperref}
\usepackage{url}

\usepackage{xspace}

\usepackage{subcaption}
\usepackage{multirow}
\usepackage{algorithm}
\usepackage[table]{xcolor} % 支持颜色
\usepackage{booktabs}      % 美化表格线
\usepackage{array}         % 支持更灵活的列定义
\usepackage{pifont}        % 提供 \ding 对勾叉号符号
\usepackage{makecell}  % 允许表格单元格内换行
\usepackage{adjustbox}  % 自动缩放表格宽度但不缩得太小
\usepackage{tikz}
\usepackage{soul}
\usepackage{xcolor}
\usepackage{caption}
\usepackage{amssymb}
\usepackage{amsmath} 
\usepackage{booktabs}
\usepackage{enumerate}
\usepackage{graphicx}
\usepackage{subcaption}
\usepackage{xspace}
\usepackage{float}
\usepackage{bbm}
\usepackage{bm}
\usepackage{multirow}
\usepackage{booktabs}
\usepackage{color}
\usepackage{framed}
\usepackage{stfloats}
\usepackage{iitem}
\usepackage[T1]{fontenc}
\definecolor{shadecolor}{RGB}{180,180,180}
\usepackage{url}
\newcommand{\ignore}[1]{}

\usepackage{algpseudocode}
\usepackage{amsthm}
\usepackage{amsmath}
\usepackage{tikz}

\usepackage{amsmath}
\usepackage{colortbl}
\usepackage{color, xcolor}
\definecolor{reviseorange}{RGB}{180,90,0}
\usepackage{wrapfig} % 在导言区添加

\newtheorem{proposition}{Proposition}
\newtheorem{theorem}{Theorem}
\newtheorem{lemma}{Lemma}
\newtheorem{corollary}{Corollary}

\newtheorem{assumption}{Assumption}

\title{GraphPO: Graph-based Policy Optimization for Reasoning Models}

\author{%
\textbf{Yuliang Zhan}$^{1,\ast}$ \quad
\textbf{Xinyu Tang}$^{1,\ast}$ \quad
\textbf{Jian Li}$^{1}$ \quad
\textbf{Dandan Zheng}$^{2}$ \quad
\textbf{Weilong Chai}$^{2}$ \\
\textbf{Jingdong Chen}$^{2}$ \quad
\textbf{Jun Zhou}$^{2}$ \quad
\textbf{Ge Wu}$^{2}$ \quad
\textbf{Wenyue Tang}$^{2}$ \quad
\textbf{Hao Sun}$^{1}$
\\[1mm]
$^{1}$ Gaoling School of Artificial Intelligence, Renmin University of China 
$^{2}$ Ant Group \\
}
\begin{document}

\maketitle

\begingroup
\renewcommand{\thefootnote}{\fnsymbol{footnote}}
\footnotetext[1]{Equal contribution.}
\footnotetext[2]{Corresponding author.}
\endgroup

\begin{abstract}
Reinforcement Learning with Verifiable Rewards (RLVR) has become a standard paradigm for enhancing the capability of large reasoning models. 
RLVR typically samples responses independently and optimizes the policy using from final answers. This paradigm has two limitations. First, independently responses often contain similar intermediate reasoning steps, causing redundant exploration and wasted computation. Second, sparse final-answer rewards make it hard to identify useful steps. Tree-based methods partly address this problem by sharing prefixes and comparing branches from the same prefix to provide fine-grained signals. However, tree branches are still expanded independently. When different branches reach similar reasoning states, they cannot share information and repeat similar exploration. Moreover, tree-based methods ignore such dispersion and only perform local comparisons within separate branches, which can lead to higher variance in advantage estimation.
To address this challenge, we propose \textbf{GraphPO} (\textbf{Graph}-based \textbf{P}olicy \textbf{O}ptimization), a novel RL framework that represents rollouts as a directed acyclic graph, with reasoning steps as edges and semantic states summarized from the reasoning paths as nodes. GraphPO merges semantically equivalent reasoning paths into equivalence classes, allowing them to share suffixes and reallocating budget away from redundant expansions to diverse exploration. Furthermore, we assign efficiency advantages to incoming edges and correctness advantages to outgoing edges, thereby improving inference efficiency while deriving process supervision from outcome. Theory shows that GraphPO reduces advantage-estimation variance and enhances reasoning efficiency.
Experiments on three LLMs across reasoning and agentic search benchmarks show that GraphPO consistently outperforms chain- and tree-based baselines with the same token budgets or response budgets. The code will be released soon.
\end{abstract}

\section{Introduction}\label{sec:intro}
% This is introduction~\citep{zhan2024over,tang2025enhancing,tang2025rethinking}
% RL对LLM推理很有用。现在的RL方法常用RLVR（Reinforcement Learning with Verifiable Rewards）方法，使用输出来让模型学习答案的回答。但是这种基于输出奖励稀疏，
% 进行cridit assignment时很难知道奖励哪一步，惩罚哪一步。而且
% 为了解决

% 最近，在后训练中使用Reinforcement Learning（RL）使得推理模型在agent、coding、math领域取得了里程碑式的进展。很多大推理模型使用RLVR通过利用最终答案的可验证正确性作为奖励信号来增强推理能力。然而，RLVR方法面临着两个核心挑战：首先，最终答案奖励本质上是稀疏的，因此模型在优化过程中进行信用分配时难以指出哪些中间步骤是有效。其次，现有方法通常独立采样和评估多条推理轨迹，未能充分复用不同轨迹中相似的中间推理状态，导致探索冗余较高，训练效率受限。
Recently, Large Reasoning Models~(LRMs) have achieved milestone progress in agentic, coding, and mathematical reasoning tasks~\citep{guo2025deepseek,team2025kimi,yang2025qwen3}. 
They typically adopt reinforcement learning with verifiable rewards~(RLVR), where the correctness of the final answer is used as a binary reward to optimize the policy~\citep{becker2025troll,zhao2025geometric,yu2025dapo}. 
Effective RL requires both accurate reward credit assignment for policy updates and diverse exploration during sampling.
However, current RLVR methods still struggle with both. First, outcome rewards are inherently sparse, which hinders credit assignment to intermediate reasoning steps that contribute to the final answer~\citep{kazemnejad2024vineppo,setlurrewarding}. 
Second, RLVR methods typically sample response independently, where different responses often contain repeated intermediate reasoning steps. This leads to substantial sampling redundancy and limits exploration diversity~\citep{yang2025treerpo,huang2026pros}.

% 为了克服上述挑战，现有方法主要围绕步骤级监督展开。过程奖励模型（process reward models, PRMs）可以为中间推理步骤提供密集奖励，从而缓解稀疏结果奖励下的信用分配问题，但它们通常依赖昂贵的步骤级标注，并且跨领域泛化能力有限。搜索式方法可以在推理时估计中间状态价值，但往往计算开销较大，且难以自然融入端到端 RL 训练。近期，一些方法尝试在 RL 框架内直接从结果奖励中推导步骤级信号，例如隐式过程奖励，从而减少对人工标注和外部搜索的依赖。然而，这些方法大多仍基于独立采样的 reasoning trajectories，无法复用不同轨迹中相似的中间状态。为了同时改善信用分配和探索效率，tree-based RL 方法将推理过程从独立链式轨迹扩展为结构化 rollout，使不同分支能够共享前缀，并在分支点形成更自然的步骤级对比信号。
For accurate reward credit assignment, existing methods introduce step-level signals. Process Reward Models (PRMs)~\cite{shen2026let,wang2024math,zhang2025lessons} reward intermediate reasoning steps, helping the policy identify steps that contribute to the final answer. However, they usually require costly step-level annotations and often generalize poorly across domains. 
Furthermore, recent methods avoid explicit PRMs by estimating intermediate-state values, such as search-based estimation~\cite{chen2024alphamath,qimutual} and extraction of process signals from outcome rewards~\cite{cui2025process,fenggroup}. Although these methods alleviate reward sparsity by providing denser supervision signal, they still rely on chain-based rollouts and cannot capture similarity between independent rollouts(Figure~\ref{fig:Rollout}a), which limits both credit assignment and exploration efficiency. To address this challenge, \emph{tree-based RL methods} organize rollouts as trees~\cite{yang2025treerpo,huang2026pros,yizhitreepo,ji2025tree,wang2025scheduling,hou2025treerl,guosegment}, where shared prefixes avoid repeated reasoning and branching points provide fine-grained supervision (Figure~\ref{fig:Rollout}b).

% 然而，tree-based 方法仍存在三点局限。首先，树结构会倾向于扩展更多分支来提高成功率，容易在 post-training 后诱导模型生成冗余中间步骤，从而降低推理效率。其次，提前停止的分支难以提供有效的策略更新信号，限制了优化效率。第三，由于子树的独立，语义等价的中间推理状态在不同子树中仍被独立处理，导致冗余探索仍然存在。这三个局限并非独立缺陷，而是同一结构性盲区的不同表现：树结构强制分支独立，因此原理上无法表达"推理汇聚"——即不同推理路径经由不同中间步骤到达语义等价状态这一普遍现象。树结构无法感知汇聚，因此既不能从中提取更多的监督信号，也不能利用它消除冗余，更不能将其用于效率优化。这一分析指向一个根本性的范式转变：推理过程的忠实数学抽象不是树，而是有向无环图。
Despite their benefits, tree-structured rollouts still have three limitations. First, deeper trees give uncertain intermediate states more chances to reach correct answers, but this optimization may encourage redundant reasoning steps and \emph{\textbf{reduce inference efficiency}}.
Second, early-stopped path provide limited training signals, \emph{\textbf{reducing exploitation efficiency}}. 
Third, tree only share early prefixes. After branches diverge,  they may still reach similar intermediate states and expand them independently, causing redundant exploration and \emph{\textbf{limiting exploration efficiency}}, as analyzed in Section~\ref{sec:Empiricalstudy}.
These limitations arise from the same structural constraint: trees treat branches as independent paths. Therefore, they cannot represent equivalent states reached by different paths, compare the efficiency of these paths, or reuse subsequent reasoning after these states. Based on this observation, we model rollout as a directed acyclic graph, which explicitly aggregates equivalent reasoning states.

\begin{wrapfigure}{r}{0.5\textwidth}  % r表示图片在右边；l为左边
  \centering
  \vspace{-10pt}
  \includegraphics[width=0.5\textwidth]{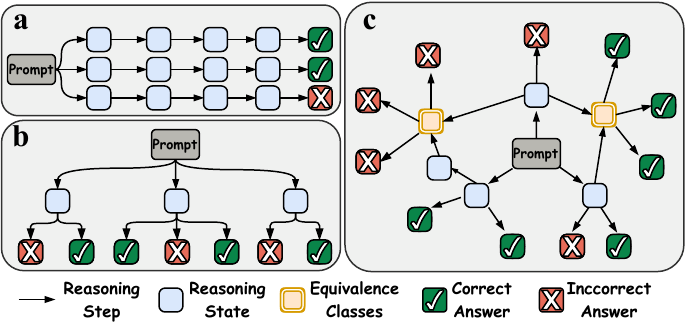}
    \vspace{-13pt}
  \caption{Comparison of rollout strategies. \textbf{(a)} Chain rollouts sample independent trajectories. \textbf{(b)} Tree rollouts share prefixes. \textbf{(c)} Graph rollouts merge semantically similar states.}
  \label{fig:Rollout}
    \vspace{-10pt}
\end{wrapfigure}
% 基于这一原理，我们提出GraphPO——一个将推理rollout结构重构为DAG的全新RL框架。
Based on this insight, we propose \textbf{GraphPO} (\textbf{Graph}-based \textbf{P}olicy \textbf{O}ptimization), a novel RL framework that represents reasoning rollouts as an RL graph, where edges denote generated reasoning steps and nodes denote the intermediate semantic states summarized from the reasoning path from the initial prompt (Figure~\ref{fig:Rollout}c).
During graph-based rollout, GraphPO incrementally constructs the RL graph by detecting semantically similar states across paths and virtually merging similar reasonings into equivalence classes. These equivalence classes allow GraphPO to \emph{improve inference}, \emph{exploitation}, and \emph{exploration efficiency} through path comparison, signal sharing, and redundancy reduction.
First, paths reaching the equivalence class solve the similar sub-problem. GraphPO therefore introduces a path-level efficiency advantage that rewards shorter paths within each class, steering the learned policy toward more efficient reasoning.
Second, reasoning in the same equivalence class share the same semantic state, so their subsequent correctness samples can be shared. GraphPO therefore improves exploitation efficiency by sharing subsequent correctness signals among similar reasonings, which produces denser step level rewards. Even early-stopped branches can receive credit from equivalent partners, improving rollout utilization and stabilizing policy updates.
Third, when a path reaches an already discovered equivalence class, further expansion is likely redundant. GraphPO reduces its next-layer budget and reallocates the saved computation to novel frontier states, encouraging broader exploration under the same token budget. As analyzed in Section~\ref{sec:Empiricalstudy}, this effectively improves exploration efficiency.

%  We introduce \textbf{GraphPO}, a novel RL framework with graph-based rollout strategy，仅通过answer可以得到self-Emergent Process Supervision。
$\bullet$ We introduce GraphPO, a graph-based RL framework that merges semantically similar reasoning to reduce redundant exploration, improve rollout utilization and get step-level rewards.

% 基于通过detecting semantically similar states的graph-based rollout，我们提出dual-group graph advantage。dual-group graph advantage由入边的efficiency advantage和出边的correctness advantage构成，能够有效的提升推理效率、探索效率。
$\bullet$ We propose a dual-group graph advantage estimation method that improves reasoning efficiency by comparing incoming paths within each equivalence class and improves reasoning performance by comparing outgoing steps at each node.

% 我们提供了充分的理论和实验，证明我们的方法能够在更少的rollout budgets（tokens）的情况下取得更好的性能
$\bullet$ Extensive experiments and theoretical analyses demonstrate that GraphPO achieves better performance with the same rollout budgets or response budgets.

\section{Related Work}
\label{sec:relatedwork}
\textbf{Reinforcement Learning with Sparse Supervision.}
RLVR drives post-training of large reasoning models with a binary outcome reward~\citep{guo2025deepseek,team2025kimi,yu2025dapo}, but this signal is too sparse for critical reasoning step of credit assignment~\citep{kazemnejad2024vineppo,wen2026rlvrimplicit}. A common remedy is to densify supervision with process reward models~\citep{shen2026let,wang2024math,zhang2025lessons,zou2025reasonflux}, but they require costly annotations and transfer poorly across domains. Recent work reduces this dependence by deriving process signals from outcomes through value estimation~\citep{chen2024alphamath,qimutual}, implicit rewards~\citep{cui2025process,liu2025trust,han2026selfaligned}, or segment-level credit assignment~\citep{setlurrewarding,fenggroup,guosegment}. These methods estimate supervision from independent rollouts. Thus, outcomes from semantically equivalent states cannot support each other. GraphPO merges such states into equivalence classes, enabling suffix sharing that provides denser, lower-variance samples for credit assignment.
% GraphPO可以通过合并等价状态来共享后缀，从而生成方差更低的过程监督信号。

\textbf{Tree Search for Reinforcement Learning.}
Tree-structured rollouts share prefixes across expansions, allowing them to generate more trajectories than conventional RLVR methods under the same token budget. The branching points in these trajectories naturally provide step-level comparisons.~\citep{yang2025treerpo,yizhitreepo,ji2025tree,hou2025treerl}. Recent variants further improve efficiency by reusing common prefixes, scheduling branch expansion toward informative states, or extending frontier branches through lookahead~\citep{huang2026pros,wang2025scheduling,xing2025lookahead,shi2025montecarlo,huang2025treeopo}. 
These methods share prefixes, but their tree topology still treats diverged branches independently, even when they later reach semantically equivalent states. GraphPO instead merges such states into a DAG, enabling length comparison, signal sharing, and budget reallocation within each equivalence class to improve reward exploitation and exploration efficiency.
\section{Empirical Study}
\label{sec:Empiricalstudy}
% In this section, we present our empirical study.我们分析不同的采样策略的推理冗余和探索效率。
In this section, we analyze reasoning redundancy and exploration efficiency of rollout strategies.

\textbf{Experimental Setup.} 
% 我们follow PROS~\citep{huang2026pros}，在不同长度前缀窗口下计算64条推理路径的语义相似度.我们对比当前前缀与其他更短前缀。为了消除表面形式的干扰，提取语义本质，我们使用Qwen2.5-7B-Instruct~\cite{qwen2025qwen25technicalreport}进行推理路径总结。然后我们follow MIRB~\citep{ju2025mirb}，使用SFR-Embedding-2\_R~\citep{meng2024sfr}计算不同总结之间的语义相似度.我们分析了独立采样、树结构采样和图结构采样三种不同的采样策略。我们在数据集MATH500~\citep{hendrycksmath2021}数据集上进行分析.
Following PROS~\citep{huang2026pros}, we sample 64 reasoning trajectories for each prompt. To reduce expressive variation and better capture underlying semantics, we summarize intermediate reasoning states with Qwen2.5-7B-Instruct~\citep{qwen2025qwen25technicalreport} and measure summary similarity using SFR-Embedding-2-R~\citep{meng2024sfr}, following MIRB~\citep{ju2025mirb}. We use MATH500~\citep{hendrycksmath2021} and compare independent chain-based, tree-based, and graph-based sampling strategies, using Qwen2.5-7B-Math~\citep{yang2024qwen2} for sampling.

\begin{figure*}[t]
    \centering
    \includegraphics[width=\textwidth]{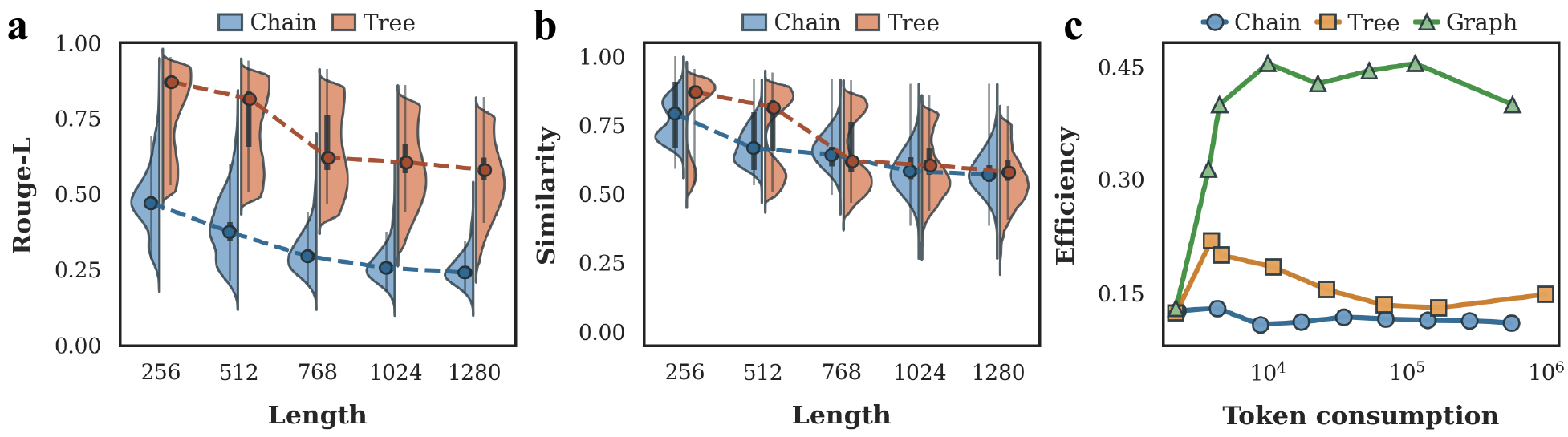}
    % Semantic Redundancy和探索效率的empirical study。（a）语义相似度和窗口大小的关系。(b)Pairwise normalized edit distance随窗口大小变化。越高，相似度越高。（c）不同采样策略的探索效率比较。
    \caption{Empirical study of semantic redundancy and exploration efficiency. \textbf{(a)} Pairwise normalized edit similarity across prefix window sizes. \textbf{(b)} Semantic similarity across prefix window sizes. \textbf{(c)} Exploration efficiency across different sampling strategies. For (a) and (b), wider regions in each violin indicate denser pairwise score distributions.}
    \vspace{-13pt}
\label{fig:empirical}
\end{figure*}
\textbf{Chain and Tree Structured Sampling Contain Substantial Semantic Redundancy.}
% 图 2a 和图 2b 显示，许多前缀与其他路径中的较短前缀具有很高的语义相似度（similarity > 0.9）。这说明，基于链的 rollout 还是基于树的 rollout，都经常重复访问相似的中间推理状态。
% 为了展示不同采样策略产生冗余。我们对每一个prompt的tolloutextract prefixes under different window lengths. For each prefix, we compute its semantic similarity and Rouge-L~\citep{lin2004rouge} with shorter prefixes from other trajectories。我们在图~\ref{fig:empirical}a和b中展示了the distribution of pairwise semantic similarity and Rouge-L between prefixes and shorter prefixes from other trajectories under different prefix window sizes.
To show the redundancy produced by different sampling strategies, we extract rollout prefixes under different window sizes for each prompt. For each prefix, we compute its semantic similarity and Rouge-L~\citep{lin2004rouge} score with shorter prefixes from other trajectories. Figures~\ref{fig:empirical}a and~\ref{fig:empirical}b show the distributions of pairwise semantic similarity and Rouge-L scores under different prefix window sizes, where many prefixes are highly similar to prefixes from other paths (similarity > 0.9). This suggests that rollouts based on chains and trees both often revisit similar intermediate reasoning states.

% 详细来说，由图~\ref{fig:empirical}a可以看出，chain采样的到的shorter prefixes yield markedly higher similarity, indicating substantial redundancy in repeatedly generating nearduplicate initial reasoning steps. Tree 通过共享前缀的方法，降低了早期推理步骤的冗余，从而通过更少的token预算实现更高的探索diverse。但是Figure~\ref{fig:empirical}b显示，Tree采样的到的prefixes之间的语义相似度仍然很高，和chain采样的prefixes之间的相似度并没有明显降低。这说明树结构虽然通过共享祖先减少了一部分重复计算，但它仍然会在不同分支中反复展开相似的中间推理状态。特别是，当两个分支在语义上已经到达相似状态时，树结构仍然将它们视为独立节点，并继续分别扩展它们。随着树的加深，共享前缀的影响逐渐减弱，高相似度区域也向低相似度区域移动，但这种跨分支冗余仍然存在。
Specifically, Figure~\ref{fig:empirical}a shows that chain sampling produces highly similar shorter prefixes, indicating substantial redundancy in early reasoning steps. Tree-based sampling reduces this early redundancy by sharing prefixes, thus improving exploration diversity under the same token budget.
However, Figure~\ref{fig:empirical}b shows that tree prefixes still have high semantic similarity, with similarity scores close to chain sampling. This suggests that the tree structure reduces part of the repeated computation by sharing ancestors, but it still independently reach and expand semantically similar reasoning, causing redundant exploration. In particular, when two branches have reached semantically similar states, the tree structure still treats them as independent nodes and continues to expand them separately.
% % 同时链式和树式展开表现出不同的冗余模式。在链式展开中，图a显示较短前缀存在两个主要相似度区域：一部分前缀与其他路径高度相似，另一部分相似度较低。随着前缀长度增加，高相似度区域逐渐减弱，大部分得分转向低相似度区域。这说明独立采样的链式轨迹可能在早期访问相似的中间状态，但这些轨迹随后会逐渐分化。
% % 相比之下，树式展开在不同前缀长度下都保留了明显的高相似度区域，因此图b始终呈现两个主要集中区域。这表明树结构虽然通过共享祖先减少了一部分重复计算，但它仍然会在不同分支中反复展开相似的中间推理状态。特别是，当两个分支在语义上已经到达相似状态时，树结构仍然将它们视为独立节点，并继续分别扩展它们。随着树的加深，共享前缀的影响逐渐减弱，高相似度区域也向低相似度区域移动，但这种跨分支冗余仍然存在。
% Meanwhile, chain rollouts and tree rollouts exhibit different redundancy patterns. In chain rollouts, short prefixes contain two clear groups of scores: some are highly similar to prefixes from other trajectories, while others are much less similar. As the prefix length increases, most scores shift toward the lower-similarity region. This suggests that independently sampled chains may overlap early, but gradually diverge later.
% In contrast, tree rollouts retain a clear high similarity region across different prefix lengths because many branches share early ancestors. However, shared ancestors only reduce redundancy along common prefixes. Semantically similar states in different branches are still treated as separate nodes and expanded independently. 

These observations motivate a graph-based rollout structure. Instead of only sharing exact prefixes, Graph can connect semantically similar intermediate states across different branches. In this way, later rollouts can reuse information from related states and adaptively reallocate rollout budgets, rather than expanding each similar state in isolation. This graph structure reduces redundant exploration and provides a more fine-grained basis for credit assignment across reasoning steps.

\textbf{Graph Structured Turns Semantic Redundancy into Exploration Efficiency.}
% 证明确实可以增加探索效率
% 证明探索效率（正确步骤*token/token，横轴是采样策略，纵轴是探索效率，柱状图）
% 为了证明图结构采样可以将冗余推理步骤转化为更有效的探索。
To examine whether graph structured sampling converts redundancy into useful exploration, we measure exploration efficiency under different token budgets. As the budget increases, chain sampling adds more independent rollouts, whereas tree and graph sampling increase the branching factor and depth. For the current $n$ rollouts, let $\mathcal{S}_n$ denote all reasoning steps, and let $|s|$ be the token cost of step $s$. For tree and graph rollouts, a step is correct if it can lead to a correct final answer. For chain rollouts, each full trajectory is treated as one step. We define exploration efficiency as:
\begin{equation}
    \mathrm{Eff}(n)=
    \frac{\sum_{s\in \mathcal{S}_n}\mathbf{1}[s\ \text{is correct}]\,|s|}
    {\sum_{s\in \mathcal{S}_n}|s|}.
    \label{eq:eq1}
\end{equation}
A higher value indicates that more computation is allocated to useful reasoning progress rather than redundant expansions.
Figure~\ref{fig:empirical}c shows that chain-based sampling remains inefficient as token consumption grows, since independent chains repeatedly explore overlapping or unproductive states. Tree-based sampling improves efficiency at small budgets through prefix sharing, but its advantage decreases with larger budgets. This suggests that later branches still expand semantically similar states independently. In contrast, graph-based sampling achieves substantially higher efficiency across token budgets. Its efficiency rises quickly after the initial rollout stage and remains consistently above chain- and tree-based sampling. This supports our motivation that detecting and merging semantically similar states redirects rollout budget from redundant reasoning to more useful exploration.

\section{Method}
\label{sec:method}
\begin{figure*}[t]
    \centering
    \includegraphics[width=\textwidth]{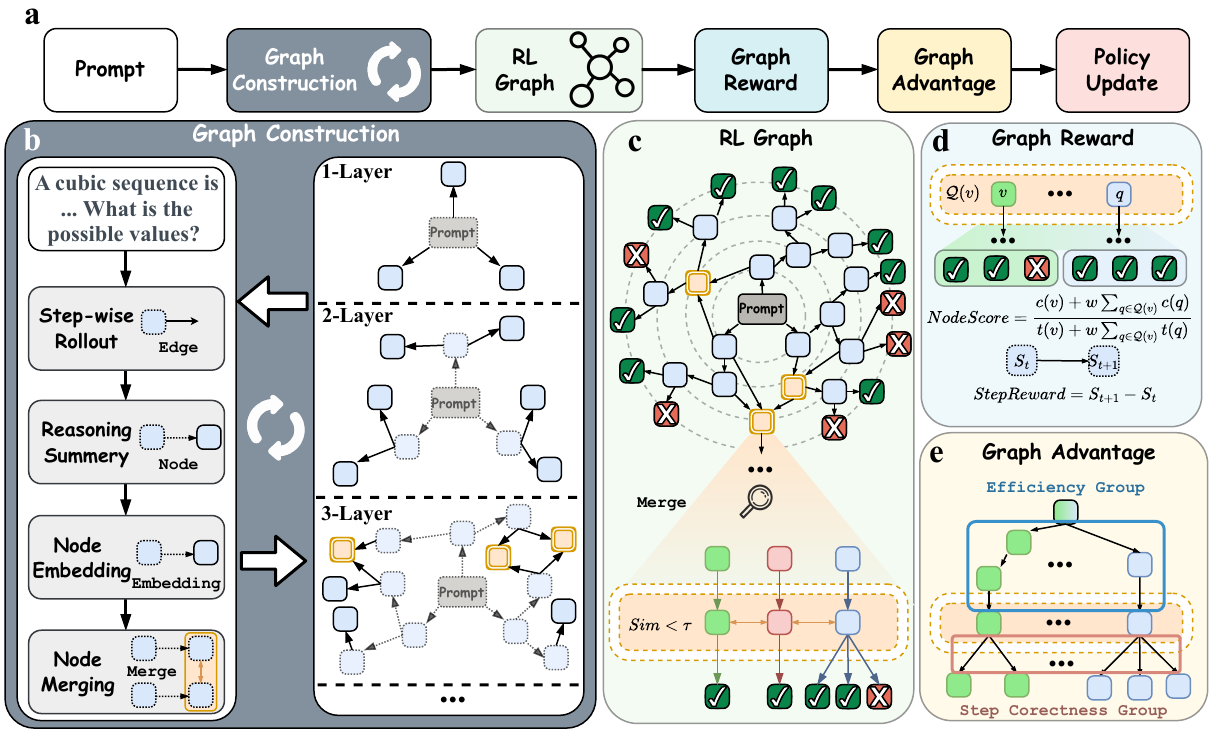}
    \caption{Overview of GraphPO. \textbf{(a)} Overall model architecture. \textbf{(b)} Graph Construction via Step-Level Expansion. \textbf{(c)} The resulting RL graph shares downstream successors among merged states. \textbf{(d)} Scoring Nodes and Deriving Graph Step Rewards. \textbf{(e)} Dual-Group Graph Advantage.}
    \vspace{-13pt}
    \label{fig:main}
\end{figure*}
% 中文译文 本节介绍 GraphPO。我们先概述整体方法，再介绍图构建、图奖励和图优势。
In this section, we describe Graph-based Policy Optimization. We first delineate the overall methodology, then detail graph construction, graph reward, graph advantage and policy optimization. 

\subsection{Overview}
% 中文译文 基于上述分析，我们提出 GraphPO，一个基于图采样的强化学习框架。GraphPO 将 rollout 组织为有向无环图，其中每个采样的推理 step 形成一条边，执行该 step 后得到的完整 reasoning path 被总结并嵌入为节点状态，用于后续的等价检测和合并。
Based on the above analysis, we propose GraphPO, a novel reinforcement learning framework with graph-based sampling. GraphPO incrementally constructs a directed acyclic graph $\mathcal{G}=(\mathcal{V},\mathcal{E})$ step by step, with sampled reasoning steps as edges and nodes represent intermediate semantic states summarized from the reasoning path starting at initial prompt and ending at current step. 
GraphPO then embeds these node states for equivalence detection. GrphPO then compares node embeddings across paths and layers, and merges semantically similar states into equivalence classes, yielding a graph with both direct and shared edges (Section~\ref{sec:graph_construction}). Equivalence classes reduce redundant exploration by reallocating rollout budgets among semantically equivalent paths. 
Based on this graph, equivalence classes also enable reuse of reasoning suffixes across similar states, converting sparse outcome rewards into dense step-level edge rewards for credit assignment (Section~\ref{sec:graph_reward}). We further introduce a dual-group graph advantage for credit assignment and efficiency learning (Section~\ref{sec:graph_advantage}). The correctness group compares direct edges and shared steps from similar state, and the efficiency group compares paths that reach the same merged state. An
overview is depicted in Figure~\ref{fig:main}. The algorithm is presented in Appendix~\ref{app:algorithm}.

\subsection{Graph Construction}
\label{sec:graph_construction}
% 中文译文 只复用文本前缀会漏掉语义相同但表达不同的中间状态。GraphPO 因此逐 step 构建图，在保持原始生成上下文的同时，将每个新生成的 reasoning path 转化为可比较的语义节点。
% 和树不同，Graph需要merge相近语义的reasoning path得到equivalence  classes，从而提升推理效率、探索效率、利用效率。
Tree rollouts can reduce repeated prefixes. Howerer, The empirical study shows that when different branches reach similar reasoning states, they still explore independently and cause redundant exploration. GraphPO therefore builds a rollout graph, where edges are generated reasoning steps and nodes are the semantic states summarized from the reasoning path starting at the initial prompt. 
% By merging similar states into equivalence classes, GraphPO can improve inference, exploration, and exploitation efficiency.

% 中文译文 给定 prompt $x$，GraphPO 初始化根节点 $v_0$，并在 $h=1,\ldots,H$ 上逐层扩展。对于活跃节点 $u$，其生成上下文由从根到 $u$ 的路径上所有边 token 依序拼接得到，而不是由节点拼接得到。每个新采样的推理 step 被登记为边，新节点表示加入该 edge 后的完整路径所诱导的语义状态。
As shown in Figure~\ref{fig:main}b, GraphPO initializes the prompt as the root $v_0$ and expands the graph layer by layer. Each step ${u\to v}$ is added as a direct edge $e(u,v)$, where $u,v\in\mathcal{V}$. Following Section~\ref{sec:Empiricalstudy}, we use the policy model itself to summarize the entire edge path from $v_0$ to $v$, and node $v$ stores the summarize as semantic state $s_v$.
% 中文译文 为避免表面差异干扰汇聚检测，GraphPO 用冻结 extractor 将由 edge tokens 构成的完整 reasoning path 总结为结构化语义状态 $s_v$，再用冻结 embedder 得到节点嵌入 $\mathbf{z}_v$。节点合并基于这些 path-level 语义嵌入，而不是 token 级文本重合。
We embed each semantic state $s_v$ as $\mathbf{z}_v$ and use cosine similarity for equivalence detection. 
To preserve acyclicity, GraphPO only compares node pairs that are not in an ancestor--descendant relation. 
Two comparable nodes $u$ and $v$ are grouped into the same equivalence class when $\operatorname{sim}(u,v)=\frac{\mathbf{z}_u^{\top}\mathbf{z}_v}{\|\mathbf{z}_u\|\|\mathbf{z}_v\|}\ge \kappa$. 
Merged nodes form an equivalence class $\mathcal{Q}(v)$. Appendix~\ref{app:robustness} theoretically proves that GraphPO remains robust under imperfect equivalence detection.
% 中文译文 不同路径的真实前缀不同，物理合并会破坏策略上下文，因此采用虚拟合并。等价节点共享对方已探索出的后继，且共享边只能指向更深节点以保持 DAG。
Equivalent nodes keep their own direct context, but share suffix from other members of the same class. Let $\mathcal{C}(u)$ be the direct suffix of $u$. The graph suffix $\mathcal{C}_g(u)$ is:
\begin{equation}
    \mathcal{C}_g(u)=\mathcal{C}(u)\cup\{c\mid q\in \mathcal{Q}(u)\setminus\{u\},\, c\in\mathcal{C}(q)\}.
\end{equation}
% 中文译文 当一个节点经更长路径到达已知等价类时，其后续计算更冗余。GraphPO 在保留至少一个续写的前提下减半其采样预算，把计算转移到更短或尚未汇聚的状态。
As shown in Figure~\ref{fig:main}c, when a path reaches an existing equivalence class, further expansion is likely redundant. GraphPO halves its next-layer budget while keeping at least one continuation and reallocates the saved budget to other paths.
% 中文译文 共享后继通过预算重分配减少冗余探索，因此 GraphPO 将计算从已覆盖的语义区域转向更有新颖性的状态从而提高探索效率。
Thus, graph construction reduces redundant exploration and improves exploration efficiency. Appendix~\ref{app:exploration} formalizes that this improves expected correct-token efficiency and discovers more distinct semantic states than tree rollouts under the same budget.

\subsection{Graph Reward}
\label{sec:graph_reward}
% 中文译文 RLVR 只有最终答案奖励，信用分配稀疏；PRM 提供步骤级信号但需要额外标注或模型。GraphPO 利用图结构从最终答案中自发估计过程信号。
To turn sparse outcome rewards into process supervision, GraphPO builds the graph reward that shares node score in the same equivalent class. As shown in Figure~\ref{fig:main}d, it first estimates a state score for each node from verified answers, and then uses the value difference along each direct edge as the step reward. 
For node $v$, its own score comes from terminal states reachable through direct edges. 
If node $v$ belongs to an equivalence class, GraphPO pools the score of equivalent nodes because they represent similar reasoning states:
% 中文译文 为避免共享后继被重复计数，终止状态正确数仅沿原始生成边自底向上传播，得到节点自身的 Monte Carlo 统计。
% To avoid double counting through shared successors, correctness of terminal states is propagated only along original direct edges.
% \begin{equation}
%     c_v^{\mathrm{own}}=\sum_{z\in\mathcal{Z}(v)} R(z), \qquad t_v^{\mathrm{own}}=|\mathcal{Z}(v)|.
% \end{equation}
% 中文译文 等价类成员代表相近推理状态，其下游正确率可以互为证据；权重 $w<1$ 控制由错误合并带来的偏差。
\begin{equation}
\label{eq:combain}
    S(v)=\frac{c_v^{\mathrm{own}}+w\sum_{q\in\mathcal{Q}(v)\setminus\{v\}}c_q^{\mathrm{own}}}
    {t_v^{\mathrm{own}}+w\sum_{q\in\mathcal{Q}(v)\setminus\{v\}}t_q^{\mathrm{own}}},
\end{equation}
where $w\in[0,1]$ is a pooling coefficient that controls the contribution of downstream samples from equivalent nodes, $\mathcal{Q}(v)$ denotes the equivalence class containing $v$,
$c_q^{\mathrm{own}}=\sum_{z\in\mathcal{Z}(q)}R(z)$, and 
$t_q^{\mathrm{own}}=|\mathcal{Z}(q)|$. Here, $R(z)\in\{0,1\}$ is the verifier reward and 
$\mathcal{Z}(v)$ denotes the terminal states reachable from $v$ through original direct edges. When $\mathcal{Q}(v)={v}$, Eq.~\ref{eq:combain} is the the score of a non-merged node.
% It reduces to the Monte Carlo correctness rate when no partner exists.
% 中文译文 节点分数仍是状态级量。为得到步骤级奖励，GraphPO 对每条原始生成边 $e=(u,v)$ 计算两端节点的分数增益，并用两端语义相似度折扣只是复述前一状态的步骤。
To obtain a step-level reward, GraphPO computes the score gain on each direct edge $e(u,v)$ and discounts repetitive steps by endpoint similarity:
\begin{equation}
    r_{\mathrm{step}}(u,v)=\bigl(S(v)-S(u)\bigr)\bigl(1-\eta(u,v)\bigr),
\end{equation}
where $\eta(u,v)=\max(0,\operatorname{sim}(u,v))$. A step receives positive credit when it improves downstream correctness and adds new semantic content. Early-stopped branches can also receive evidence from equivalent partners. Appendix~\ref{app:variance} shows that pooling equivalent nodes yields a more accurate edge reward, and Appendix~\ref{app:prm} proves that $r_{\mathrm{step}}$ matches the expected policy-gradient direction of an oracle PRM. Thus, GraphPO obtains step-level process supervision from outcome rewards without annotated step labels.

\subsection{Dual-Group Graph Advantage}
\label{sec:graph_advantage}
% 中文译文 只在整段回答或同一前缀的局部续写之间比较，无法利用图汇聚带来的共享后继，也无法显式奖励更短的等价路径。GraphPO 因此引入双组优势。
% 基于图的结构优势，我们提出Dual-Group Graph Advantage估计。详细来说，通过到达equivalence class的reasoning path的组间比较，可以 improves reasoning efficiency ，通过node 的subsequent step组间比较，可以得到细粒度的credit assignment。
To improve reasoning efficiency and provide fine-grained credit assignment, we propose Dual-Group Graph Advantage estimation, which compares reasoning paths reaching the same equivalence class to encourage efficient reasoning, and compares subsequent steps from the same state to assign step-level credit (Figure~\ref{fig:main}e).

% 中文译文 正确性组用于比较从同一语义状态出发的候选 step。对源节点 $u$，组内边由图后继诱导得到，即 $\mathcal{E}_{\mathrm{cor}}(u)=\{(p(v),v)\mid v\in\mathcal{C}_g(u)\}$。其中共享后继仍使用产生它的原始生成边，这保证 reward 与真实生成上下文一致。
The step correctness group compares steps that \textbf{\emph{leave the same semantic state}}. For source node $u$, $\mathcal{E}_{\mathrm{cor}}(u)=\{(p(v),v)\mid v\in\mathcal{C}_g(u)\}$, where $p(v)$ indicates the predecessor node of $v$. For a shared node $v$, GraphPO computes the reward on its actual generated edge $(p(v),v)$, rather than on the virtual shared edge from $u$ to $v$. This keeps the reward tied to the real generation context. For each edge $e(p(v),v)\in\mathcal{E}_{\mathrm{cor}}(u)$, the \textbf{\emph{correctness advantage}} is:
\begin{equation}
    A_{\mathrm{cor}}(e)=
    \frac{r_{\mathrm{step}}(e)-\operatorname{mean}(\{r_{\mathrm{step}}(\hat{e})\mid \hat{e}\in\mathcal{E}_{\mathrm{cor}}(u)\})}
    {\operatorname{std}(\{r_{\mathrm{step}}(\hat{e})\mid \hat{e}\in\mathcal{E}_{\mathrm{cor}}(u)\})},
\end{equation}

% 中文译文 效率组用于比较到达同一语义状态的不同 edge paths。等价类 $Q$ 作为虚拟终点，$a_Q$ 是共享生成前驱，$\ell_v$ 是从 $a_Q$ 到 $v$ 的边数。负路径长度在组内标准化后形成 path-level advantage，再均匀分配给该路径上的边。
The efficiency group compares paths that \textbf{\emph{reach the same semantic state}}. For an equivalence class $Q$, let $a_Q$ be the deepest shared predecessor, $\ell_v$ is the number of direct edges from $a_Q$ to $v$, and $\bar{S}_Q=\sum_{v\in Q}S(v)/|Q|$. The path-level \textbf{\emph{efficiency advantage}} is:
\begin{equation}
    A_{\mathrm{eff}}(v)=\bar{S}_Q
    \frac{-\ell_v-\operatorname{mean}(\{-\ell_q\mid q\in Q\})}
    {\operatorname{std}(\{-\ell_q\mid q\in Q\})},
\end{equation}
and it is uniformly assigned to the edges on the path from $a_Q$ to $v$. The gate $\bar{S}_Q$ keeps this signal active only when the shared state has correctness support, so GraphPO prefers shorter paths only when they reach a useful semantic state. Pooling descendants from semantically equivalent paths into a single advantage group lowers the variance of $A_{\mathrm{cor}}$ (Appendix~\ref{app:variance}). Appendix~\ref{app:inference} also proves that this increases probability toward the shortest correct path and reduces expected token cost. 

\subsection{Policy Optimization}
\label{sec_policy_objective}
% 中文译文 GraphPO 用 PPO 风格裁剪目标训练；正确性优势作用在 response mask 上，效率优势只作用在等价类路径覆盖的 token 上，两者共享同一重要性比率。
% 我们最后将两个advantage combain到一起成为一个dual-group PPO objective来进行cridit assignment，从而提升推理效率并deriving process supervision from outcome
Finally, GraphPO trains the policy with a PPO-style objective. It combines correctness and efficiency into the dual-group graph advantage:
\begin{equation}
    A_{\mathrm{dual}}(e)=A_{\mathrm{cor}}(e)+\lambda_{\mathrm{eff}}A_{\mathrm{eff}}(e),
\end{equation}

where $\lambda_{\mathrm{eff}}\geq 0$ controls the efficiency signal, and $A_{\mathrm{eff}}(e)$ denotes the path-level efficiency advantage distributed onto edges along the path from $a_Q$ to $v$. Our method builds on DAPO~\citep{yu2025dapo} with:

\begin{equation}
\begin{aligned}
\mathcal{J}_{\mathrm{GraphPO}}(\theta)
={}&
\mathbb{E}_{\substack{(q,a)\sim\mathcal{D}\\ \mathcal{T}\sim\pi_{\theta_{\mathrm{old}}}(\cdot\mid q)}}
\Bigg[
\frac{1}{\textcolor{reviseorange}{|\mathcal{E}|}}
\sum_{\textcolor{reviseorange}{e\in\mathcal{E}}}
\frac{1}{\textcolor{reviseorange}{|e|}}
\sum_{\textcolor{reviseorange}{t=1}}^{\textcolor{reviseorange}{|e|}}
\bigg(
\min\Big(
r_{e,t}(\theta)\,\textcolor{reviseorange}{A_{\mathrm{dual}}(e)},\\[-1mm]
&\qquad\qquad
\operatorname{clip}\big(r_{e,t}(\theta),\,1-\varepsilon,\,1+\varepsilon\big)\,
\textcolor{reviseorange}{A_{\mathrm{dual}}(e)}
\Big)
-\beta\, D_{\mathrm{KL}}\!\left(\pi_{\theta}\,\Vert\,\pi_{\mathrm{ref}}\right)
\bigg)
\Bigg],
\end{aligned}
\end{equation}

where $r_{e,t}(\theta)=\frac{\pi_{\theta}(o_{e,t}\mid q,o_{e,<t})}{\pi_{\theta_{\mathrm{old}}}(o_{e,t}\mid q,o_{e,<t})}$ is the token-level importance ratio, $\mathcal{E}$ is the set of direct edges, and $|e|$ is the number of tokens on edge $e$. Each token on the same edge uses the same graph advantage, preserving the token-level PPO form.

\section{Experiment}
\label{sec:experiment}
In this section, we first describe the experimental setup, then we present the main results, and finally we conduct detailed analysis and ablation study.
\subsection{Experimental Setup}
\textbf{Datasets.}
% 我们在DAPO-Math dataset~\citep{yu2025dapo}数据集上进行训练。为了测试模型的推理能力，我们在AIME24, AIME25, and MATH500~\citep{hendrycks2021measuring} 三个数学数据集和two other reasoning benchmarks上GPQA~\citep{rein2023gpqa}、LiveCodeBench~\citep{jainlivecodebench}进行评估。
To evaluate reasoning ability, we conduct experiments on three mathematical reasoning benchmarks, including AIME24, AIME25, and MATH500~\citep{hendrycks2021measuring}. We also evaluate the models on two additional reasoning benchmarks, GPQA~\citep{rein2023gpqa} and LiveCodeBench~\citep{jainlivecodebench}.
To evaluate the effectiveness in LLM agentic RL, we test model on deep search tasks including General AI Assistant~\citep{mialongaia}, WebWalker~\citep{wu2025webwalker}, BrowseComp~\citep{wei2025browsecomp} and XBench~\citep{chen2025xbench}.

\textbf{Baselines.}
% 我们在Qwen2.5-7B-Math~\citep{yang2024qwen2}、Qwen3-8B-Base~\citep{yang2025qwen3}、Deepseek-R1-Distill-Qwen-7B~\citep{guo2025deepseek}三个模型上进行实验。为了对比，我们和经典的基于outcome reward的RLVR方法进行对比，包括GRPO~\citep{guo2025deepseek}和DAPO~\citep{yu2025dapo}。同时，我们还与一些基于树结构的方法进行对比，包括TreeRL~\citep{hou2025treerl}, SPO~\citep{guosegment}, TREE-GRPO~\citep{ji2025tree}, PROS~\citep{huang2026pros}, TreePO~\citep{yizhitreepo}。Detailed descriptions of these baselines are provided in Appendix~\ref{app:baseline}.
We conduct experiments on four LLMs, including Qwen2.5-7B-Instruct~\citep{qwen2025qwen25technicalreport}, Qwen2.5-7B-Math~\citep{yang2024qwen2}, Qwen3-8B-Base~\citep{yang2025qwen3}, and Deepseek-R1-Distill-Qwen-7B~\citep{guo2025deepseek}.
For comparison, we evaluate GraphPO against standard outcome-reward RLVR methods, including GRPO~\citep{guo2025deepseek} and DAPO~\citep{yu2025dapo}.
We also compare with tree-structured reasoning methods, including TreeRL~\citep{hou2025treerl}, SPO~\citep{guosegment}, TREE-GRPO~\citep{ji2025tree}, PROS~\citep{huang2026pros}, and TreePO~\citep{yizhitreepo}.
Detailed descriptions of these baselines are provided in Appendix~\ref{app:baseline}.
For GraphPO, we set the merging threshold  $\kappa$ to 0.92 and set pooling coefficient $w$ to 0.7. Train detail can be found in Appendix~\ref{app:exp_setup}.

\subsection{Main Results}
\begin{table*}[t]
\centering
\caption{Performance comparison of different methods on various reasoning benchmarks. We report Accuracy (\%) for each benchmark. The best performance for each benchmark is highlighted in \textbf{bold}.}
\resizebox{\textwidth}{!}{
\begin{tabular}{c|l|ccccc|c}
\toprule
\textbf{Model} & \textbf{Method} & \textbf{AIME24} & \textbf{AIME25} & \textbf{MATH500} & \textbf{GPQA} & \textbf{LiveCodeBench} & \textbf{Average} \\
% \midrule
% \multirow{3}{*}{\begin{tabular}[c]{@{}c@{}}\textbf{Qwen2.5-}\\\textbf{1.5B-Math}\end{tabular}}
% & Base &  &  &  &  &  &  \\
% & GRPO &  &  &  &  &  &  \\
% & DAPO &  &  &  &  &  &  \\
\midrule
\multirow{9}{*}{\begin{tabular}[c]{@{}c@{}}\textbf{Qwen2.5-}\\\textbf{7B-Math}\end{tabular}}
& Base & 13.6 & 6.3 & 54.8 & 28.5 & 5.7 & 21.8 \\
& GRPO & 23.3 & 15.5 & 79.5 & 31.8 & 11.4 & 32.3 \\
& DAPO & 25.7 & 17.2 & 82.6 & 32.3 & 12.6 & 34.1 \\
\cmidrule(lr){2-8}
& TreeRL & 26.1 & 18.9 & 84.2 & 33.6 & 12.5 & 35.1 \\
& SPO & 28.5 & 20.4 & 87.5 & 35.5 & 13.6 & 37.1 \\
& TREE-GRPO & 27.2 & 20.3 & 86.6 & 34.8 & 13.4 & 36.5 \\
& PROS & 29.1 & 21.0 & 88.1 & 36.6 & 13.5 & 37.7 \\
& TreePO & 26.8 & 19.0 & 85.5 & 34.0 & 13.0 & 35.7 \\
& GraphPO & \textbf{32.1} & \textbf{24.5} & \textbf{91.6} & \textbf{39.6} & \textbf{16.9} & \textbf{40.9} \\
\midrule
\multirow{9}{*}{\begin{tabular}[c]{@{}c@{}}\textbf{Qwen3-}\\\textbf{8B-Base}\end{tabular}} 
& Base & 29.3 & 20.8 & 79.7 & 39.1 & 22.9 & 38.4 \\
& GRPO & 30.3 & 21.3 & 80.0 & 43.4 & 27.2 & 40.4 \\
& DAPO & 32.3 & 23.8 & 82.7 & 44.0 & 27.6 & 42.1 \\
\cmidrule(lr){2-8}
& TreeRL & 32.3 & 25.1 & 83.4 & 44.4 & 28.1 & 42.7 \\
& SPO & 34.1 & 27.1 & 86.4 & 46.5 & 30.8 & 45.0 \\
& TREE-GRPO & 33.2 & 26.7 & 85.5 & 45.9 & 30.4 & 44.3 \\
& PROS & 34.3 & 27.5 & 86.9 & 47.6 & 31.1 & 45.5 \\
& TreePO & 33.1 & 25.4 & 84.7 & 45.0 & 29.4 & 43.5 \\
& GraphPO & \textbf{38.2} & \textbf{30.9} & \textbf{90.2} & \textbf{50.6} & \textbf{34.5} & \textbf{48.9} \\
\midrule
\multirow{9}{*}{\begin{tabular}[c]{@{}c@{}}\textbf{Deepseek-R1-}\\\textbf{Distill-Qwen-7B}\end{tabular}}  
& Base & 51.7 & 38.4 & 91.0 & 46.1 & 36.8 & 52.8  \\
& GRPO & 57.2 & 47.4 & 93.5 & 46.2 & 40.7 & 57.0 \\
& DAPO & 59.0 & 48.6 & 93.3 & 48.4 & 41.4 & 58.1 \\
\cmidrule(lr){2-8}
& TreeRL & 60.7 & 49.5 & 91.2 & 49.2 & 42.5 & 58.6 \\
& SPO & 62.6 & 50.7 & 92.8 & 50.5 & 44.7 & 60.3 \\
& TREE-GRPO & 61.7 & 50.5 & 92.3 & 49.9 & 44.1 & 59.7 \\
& PROS & 63.2 & 51.4 & 93.0 & 51.5 & 44.8 & 60.8 \\
& TreePO & 61.6 & 49.6 & 91.7 & 49.7 & 43.5 & 59.2 \\
& GraphPO & \textbf{66.2} & \textbf{54.7} & \textbf{96.3} & \textbf{54.5} & \textbf{48.3} & \textbf{64.0} \\
\bottomrule
\end{tabular}
}
\label{tab:main_results}
\end{table*}
% 使用相同token预算情况下：推理性能最优（大表）
% 在agent场景也可以使用。（小表）
% 训练的entrypy曲线，训练的测试集的acc曲线，训练的测试集回复长度曲线。
%
\begin{figure*}[t]
    \centering
    \includegraphics[width=\textwidth]{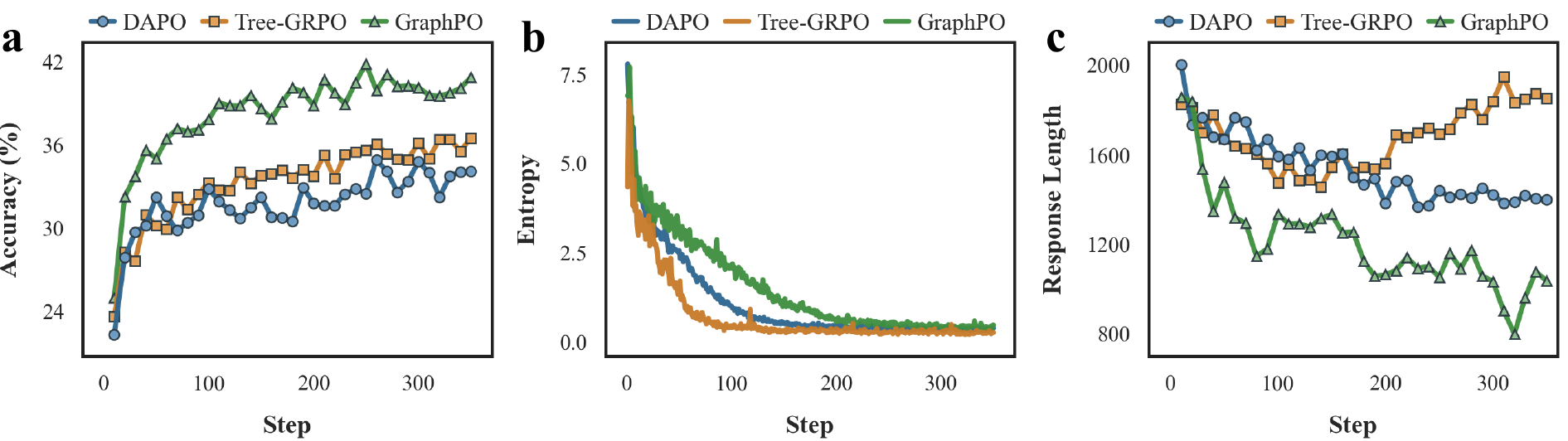}
    \caption{Training dynamics of Qwen2.5-7B-Math. \textbf{(a)} Accuracy. \textbf{(b)} Entropy. \textbf{(c)} Response length.}
    \vspace{-13pt}
    \label{fig:acc_ent_len}
\end{figure*}
% 可以发现，树方法普遍优于只依靠结果稀疏奖励的方法，这说明由outcome得到的过程信号对于提升推理性能是有效的。我们也能看到，GraphPO在所有模型和数据集上都优于树方法，这说明相比于tree structure，Graph structure能够更有效地利用过程信号来提升推理性能。这主要归功于GraphPO利用共享节点提高了rollout的利用率，同时通过polling得到了更优的过程监督。
% 图1a展示了训练过程中模型性能的变化。我们可以看到，GraphPO在训练初期就已经显著优于基线方法，并且在整个训练过程中保持领先。这表明GraphPO能够更有效地利用过程信号来提升推理性能。
\textbf{GraphPO Outperforms the Baseline Under the Same Token Budget.} Table~\ref{tab:main_results} reports model performance under the same token budget. Figure~\ref{fig:acc_ent_len}a shows the change in model performance during training. These results show that structure-based methods consistently outperform sparse outcome-reward baselines, confirming the effectiveness of outcome-derived process signals for reasoning. This improvement also comes from the fact that tree based and graph based methods can explore more reasoning paths under a same token budget, thereby collecting more trajectories. We can also see that GraphPO outperforms tree-based methods across all models and datasets. This indicates that, compared with the tree structure, the graph structure can use process signals more effectively to improve reasoning performance. This improvement mainly comes from better rollout utilization and lower-variance advantage estimation, achieved by sharing suffixes across equivalent nodes.

\textbf{GraphPO Preserves Strong Exploration. }Figure~\ref{fig:acc_ent_len}b shows the change of entropy during training. We observe that tree-based methods show the fastest entropy decay, chain-based methods are in the middle, and GraphPO decays the slowest. This is beacuse tree compare many local suffix from shared prefixes, which quickly amplifies high-reward branches and drives the policy toward a small set of reasoning patterns. In contrast, GraphPO pools downstream evidence across semantically equivalent states, allowing different paths to share support instead of competing independently. This reduces noisy over-amplification and preserves more semantically valid reasoning paths.

\begin{table*}[t]
\centering
\caption{Performance on deep search tasks for Qwen2.5-7B. The best results are indicated in \textbf{bold}.}
\label{tab:web_agent_qa_qwen7b}
\resizebox{\textwidth}{!}{
\begin{tabular}{l|cccc|cccc|c|c|c}
\toprule
\multirow{2}{*}{\textbf{Method}}
& \multicolumn{4}{c|}{\textbf{General AI Assistant}}
& \multicolumn{4}{c|}{\textbf{WebWalkerQA}}
& \textbf{BrowseComp}
& \textbf{XBench}
& \multirow{2}{*}{\textbf{Overall}} \\
\cmidrule(lr){2-5} \cmidrule(lr){6-9} \cmidrule(lr){10-10} \cmidrule(lr){11-11}
& \textbf{Lv.1} & \textbf{Lv.2} & \textbf{Lv.3} & \textbf{Avg.}
& \textbf{Easy} & \textbf{Med.} & \textbf{Hard} & \textbf{Avg.}
& \textbf{Avg.}
& \textbf{Avg.}
& \\
\midrule
ReAct       & 6.4  & 3.4  & 1.1 & 4.2  & 7.7 & 9.4  & 5.5  & 7.8  & 1.4 & 9.3 & 5.7\\
+ GRPO      & 17.2 & 14.2 & 4.3 & 14.1 & 8.7 & 11.5 & 10.7 & 10.6 & 2.3 & 10.5 & 9.4\\
+ DAPO      & 17.9 & 15.1 & 4.6 & 14.9 & 9.5 & 11.7 & 11.4 & 11.3 & 2.4 & 10.7 & 9.8\\
+ Tree-GRPO & 18.9 & 17.3 & 5.5 & 15.7 & 10.7 & 12.3 & 11.6 & 11.7 & 2.5 & 11.4 & 10.3\\
+ GraphPO   & \textbf{20.1} & \textbf{18.5} & \textbf{6.3} & \textbf{17.7}
            & \textbf{11.5} & \textbf{13.7} & \textbf{12.3} & \textbf{12.7} & \textbf{3.7} & \textbf{13.1} & \textbf{11.8}\\
\bottomrule
\end{tabular}
}
\end{table*}

\textbf{GraphPO Is Effective on Agent Tasks. }Agent tasks are naturally sequential decision making processes. They can be naturally decomposed into multiple complete decision steps. Therefore, tree and graph structures are well suited for reinforcement learning in agent tasks. To evaluate the effectiveness of GraphPO in agent scenarios, we use ReAct~\citep{yao2022react} as the agent framework. The results are shown in Table~\ref{tab:web_agent_qa_qwen7b}. GraphPO outperforms the baselines across all datasets, demonstrating its effectiveness in agent scenarios. 
\subsection{Detailed Analysis}
\textbf{GraphPO Improves Reasoning Efficiency by Reducing Redundant Paths.}
% 根据上面的进行分析
Figure~\ref{fig:acc_ent_len}c shows the response length during training. Tree-GRPO shows a non-monotonic trend, where response length first decreases and then increases. The initial decrease suggests that step-level comparisons reduce unnecessary exploration. The later increase occurs because deeper tree expansion creates more branching opportunities, which raises the probability of finding correct answers but also introduces redundant intermediate reasoning. Since the tree cannot detect or merge such redundancy, it cannot optimize repeated reasoning step, reducing inference efficiency. In contrast, GraphPO produces the shortest responses by merging semantically equivalent states and using an efficiency advantage to assign higher credit to shorter paths that reach the same state.

\textbf{GraphPO Outperforms PRM Methods Without Annotated Process Data.}
% 画个柱状图，横轴是方法，纵轴是性能。我们的方法和PRM方法进行比较。我们的方法不需要额外的标注数据，训练效率更高，性能更好。
% 在前文，我们已经证明过GraphPO的梯度更新方向和PRM方法的更新方向接近。在本section，我们进一步比较了GraphPO和PRM方法在推理性能上的差异。具体来说，我们follow ReasonFlux-PRM~\citep{zou2025reasonflux}基于Qwen2.5-7B-Instruct进行实验。具体实验结果如图~\ref{fig:detail}a所示. 可以看见，GraphPO不需要额外的标注数据仍旧优于PRM方法，这说明了从图结构中获取过程监督的合理性。
In the previous section, we have shown that the policy-gradient direction of GraphPO is close to that of PRM methods. In this section, we further compare GraphPO with PRM methods. Specifically, we follow ReasonFlux-PRM~\citep{zou2025reasonflux} and conduct experiments based on Qwen2.5-7B-Instruct. The results are shown in Figure~\ref{fig:detail}a. We can see that GraphPO still outperforms PRM methods without using additional annotated data. This shows that deriving process supervision from the graph structure is effective.

\textbf{GraphPO Outperforms the Baseline Under the Same Trajectory Budget.}
% 使用早停的策略，比较训练时间和性能。（一张柱状图，每个模型两个柱子，一个时间一个性能）（或者一个大表展示各种相同的trajectory数量下性能对比）
% 性能有提升，训练时间大幅降低。
% 在相同的token预算下，Tree和Graph结构由于共享前缀，相较于独立采样的方法能生成更多的response。在这一section，我们进一步研究生成相同数量trajectory的情况下模型的性能对比。我们让不同的方法均生成64 rollouts per prompt。结果展示在{fig:detail}b. 可以看见，在相同数量的trajectory下，GraphPO仍然优于基线方法，这说明GraphPO的性能提升不仅来自于生成更多的trajectory，更来自于更有效地利用这些trajectory进行优化。同时，我们也能看到，在相同数量的trajectory下，GraphPO的训练时间大幅降低，这说明GraphPO通过合并等价节点提高了rollout的利用率，从而减少了冗余探索，提高了训练效率。
Under the same token budget, Tree and Graph structures can generate more responses than independent sampling by sharing prefixes. We further compare different methods under the same number of trajectories, where each method generates 64 rollouts per prompt. As shown in Figure~\ref{fig:detail}b, GraphPO still outperforms the baselines, suggesting that its advantage comes not only from producing more trajectories, but also from deriving more effective process supervision. GraphPO also substantially reduces training time under the same trajectory count, indicating that merging equivalent nodes improves rollout utilization, reduces redundant exploration, and enhances training efficiency.

\textbf{GraphPO Maintains Efficient Trajectory Generation.}
% 生成trajectory速度（画个鱼泡图，大小是性能，横轴回复长度，纵轴生成速度）
% 我们进一步对比了不同方法在生成trajectory方面的效率。我们在MATH500数据集上统计了每个方法生成trajectory的平均速度和平均回复长度。The experiments are carried out on H20 GPUs without any parallelization. 结果展示在图~\ref{fig:detail}c. 可以看见，即使在rollout阶段使用了embedding 模型。GraphPO在生成trajectory方面表现出和Tree-based方法接近效率。这主要归功于GraphPO通过合并等价节点减少了冗余探索同时引导模型生成更简洁的推理路径，从而提高了生成效率。
We further compare the trajectory generation efficiency. All experiments are conducted on H20 GPUs without parallelization. As shown in Figure~\ref{fig:detail}c, GraphPO achieves trajectory generation efficiency comparable to tree-based methods, even though it uses an embedding model during rollout. This efficiency mainly comes from merging equivalent nodes, which reduces redundant exploration and encourages shorter reasoning paths.

\begin{figure*}[t]
    \centering
    \includegraphics[width=\textwidth]{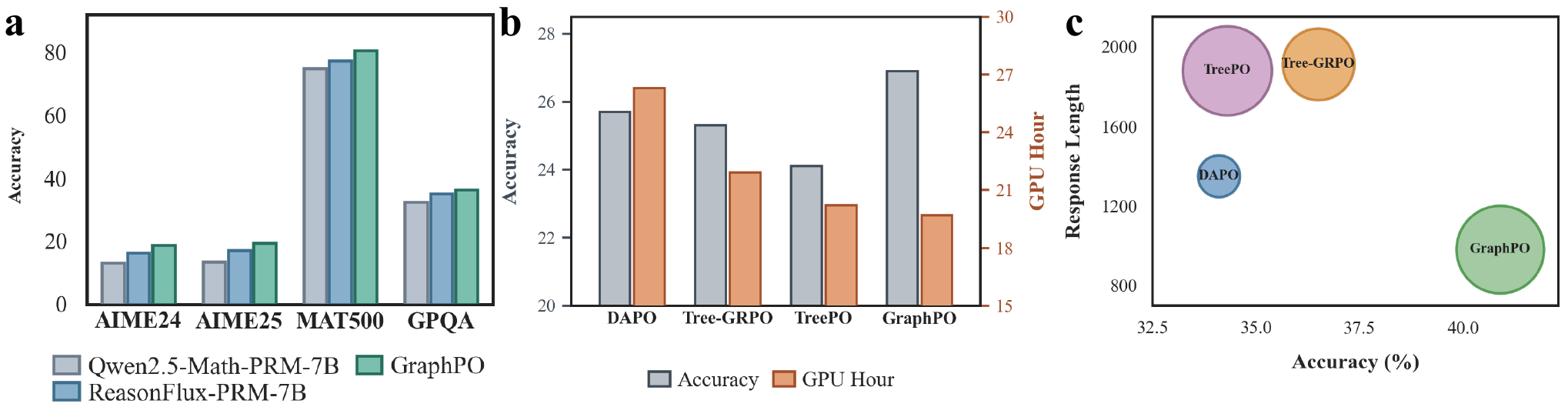}
    \caption{\textbf{(a)} PRM methods \emph{vs.} GraphPO. \textbf{(b)} Training Performance and Training Time. \textbf{(c)} Trajectory generation efficiency. The bubble size represents the number of trajectories generated per second.}
    \vspace{-13pt}
    \label{fig:detail}
\end{figure*}

\subsection{Ablation Study}
% ablation study：去掉efficiency advantage的回复长度，正确率小图，
% 合并阈值的影响（画个线图，横轴是合并阈值，纵轴是性能和效率，先升高后降低）（画训练的acc图，过大会在早期崩掉，过小退化为tree）
% 共享w值的影响
% 这三个组成一张图
% 合并阈值$\mathbf{z}_v$决定了节点合并的严格程度。In this section, we conduct ablation studies to analyze the impact of the merging threshold $\mathbf{z}_v$. 图2展示了随着阈值变化，GraphPO的性能变化。可以看见，随着合并阈值的增加，GraphPO的性能先上升后下降。这是因为较低的合并阈值会导致过度合并，抑制了合理的探索，从而降低性能；而较高的合并阈值则会导致过少的合并，无法充分利用图结构的优势，也会降低性能。当合并阈值为1时，Graph  structure退化为Tree structure，性能明显下降。
\begin{wrapfigure}{r}{0.35\textwidth}  % r表示图片在右边；l为左边
  \centering
  \vspace{-10pt}
  \includegraphics[width=0.35\textwidth]{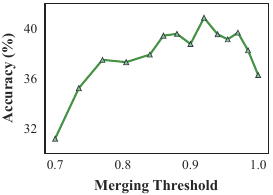}
    \vspace{-13pt}
  \caption{Performance under Different Merging Thresholds.}
  \label{fig:ablation_kappa}
    \vspace{-10pt}
\end{wrapfigure}
\textbf{Moderate Merging Thresholds Yield the Best GraphPO Performance. }The merging threshold $\kappa$ controls how strictly nodes are merged. In this section, we conduct ablation studies to analyze the impact of the merging threshold. 
Figure~\ref{fig:ablation_kappa} shows the performance under different merging thresholds. We can see that the performance of GraphPO first increases and then decreases as the merging threshold becomes larger. This is because a lower merging threshold leads to excessive merging. It suppresses reasonable exploration and thus reduces performance. In contrast, a higher merging threshold leads to too few merges. It prevents semantically equivalent reasoning paths from sharing their reasoning suffixes. When the merging threshold is 1, graph structure degenerates into a tree structure, leading to degraded performance.

% 去掉efficiency advantage的回复长度，正确率小图，
% 共享w值的影响（画训练的acc图，过大会在早期崩掉，过小退化为tree）

\section{Conclusion}
In this work, we propose GraphPO, a graph-based reinforcement learning framework for improving the capability of large reasoning models under verifiable outcome. Motivated by the semantic redundancy observed in chain- and tree-based rollouts, GraphPO represents rollouts as directed acyclic graph, where generated reasoning steps are edges and path-defined semantic states are nodes. By detecting and virtually merging semantically similar intermediate states, GraphPO reuses suffixes, reduces redundant expansions, and converts sparse outcome rewards into dense step-level process supervision. We further introduce a dual-group graph advantage in which the correctness group compares shared outgoing steps, while the efficiency group compares different incoming paths that reach the same semantic state, encouraging the policy to favor shorter paths toward useful reasoning states. Theoretical analysis shows that next-step sharing reduces advantage-estimation variance and the efficiency advantage favors shorter correct reasoning. Experiments across three large reasoning models, standard reasoning benchmarks, and agentic search tasks demonstrate that GraphPO consistently outperforms outcome-only and tree-based RL baselines under matched budgets, while maintaining efficient trajectory generation and improving response efficiency. These results suggest that graph-structured rollouts provide a practical way to obtain self-emergent process supervision from outcome rewards alone.

\bibliographystyle{unsrt}
\bibliography{neurips_2026}

%%%%%%%%%%%%%%%%%%%%%%%%%%%%%%%%%%%%%%%%%%%%%%%%%%%%%%%%%%%%
\newpage
\appendix
%%%%%%%%%%%%%%%%%zhanyuliang%%%%%%%%%%%%%%%%%%%%
\hypersetup{pageanchor=false} % avoid duplicate page anchors after appendix page reset
\setcounter{figure}{0}
\setcounter{table}{0}
\setcounter{page}{1}% 重置公式编号
\setcounter{equation}{0}
\setcounter{algorithm}{0}

\renewcommand{\theequation}{S.\arabic{equation}}
\renewcommand{\thefigure}{S.\arabic{figure}}
\renewcommand{\thetable}{S.\arabic{table}}
\renewcommand{\thealgorithm}{S.\arabic{algorithm}}
\renewcommand{\theHequation}{S.\arabic{equation}}
\renewcommand{\theHfigure}{S.\arabic{figure}}
\renewcommand{\theHtable}{S.\arabic{table}}
\renewcommand{\theHalgorithm}{S.\arabic{algorithm}}

%%%%%%%%%%%%%%%%%zhanyuliang%%%%%%%%%%%%%%%%%%%%
{\centering \Large \textbf{APPENDIX}}

\section{Algorithm}\label{app:algorithm}

We present the complete procedure of GraphPO in Algorithm~\ref{alg:graphpo}, which integrates the three stages introduced in the main text into a unified training loop.

\begin{algorithm}[H]
\caption{GraphPO Graph-based Policy Optimization}
\label{alg:graphpo}
\begin{algorithmic}[1]
\Require Policy $\pi_\theta$, prompts $\mathcal{D}$, segment length $L$, max layers $H$, samples per state $b$, convergence threshold $\kappa$, evidence-pooling weight $w$, efficiency coefficient $\lambda_{\mathrm{eff}}$
\For{each training iteration}
    \State Sample prompt $x\sim\mathcal{D}$; initialize root semantic state with $x$
    \Statex \hspace{\algorithmicindent}\textcolor{gray}{\(\triangleright\) \textit{Stage 1 Graph Construction (Section 3.2)}}
    \For{layer $h=1,\dots,H$}
        \For{each active semantic state $u$ at layer $h-1$}
            \State Sample $b_u$ continuations from $\pi_\theta(\cdot\mid \operatorname{prefix}(u))$; split them into segment-level transitions
            \State Register each new state $v$ with generation predecessor $p(v)=u$
            \State Extract structured summary and compute embedding $\mathbf{z}_v$ for each new state $v$
        \EndFor
        \For{each non-causal state pair $(u,v)$}
            \If{$\operatorname{sim}(u,v)\ge\kappa$}
                \State Union-Find links $u$ and $v$ into the same semantic equivalence class
            \EndIf
        \EndFor
        \State Build graph successors $\mathcal{C}_g(v)$ from direct successors and convergent-state successor sharing
        \State Reallocate the saved budget from longer routes to active non-redundant frontier states
    \EndFor
    \Statex \hspace{\algorithmicindent}\textcolor{gray}{\(\triangleright\) \textit{Stage 2 Graph Reward (Section 3.3)}}
    \State Evaluate terminal states with verifier and obtain $R(z)\in\{0,1\}$
    \State Propagate $c_v^{\mathrm{own}}, t_v^{\mathrm{own}}$ bottom-up along direct generation-provenance successors
    \State Compute node scores $S(v)\leftarrow\bigl(c_v^{\mathrm{own}}+w\textstyle\sum_{q\in\mathcal{Q}(v)\setminus\{v\}}c_q^{\mathrm{own}}\bigr)/\bigl(t_v^{\mathrm{own}}+w\textstyle\sum_{q\in\mathcal{Q}(v)\setminus\{v\}}t_q^{\mathrm{own}}\bigr)$ when the denominator is positive
    \State Compute edge rewards $r_{\mathrm{step}}(p(v),v)\leftarrow (S(v)-S(p(v)))\cdot(1-\eta(p(v),v))$
    \Statex \hspace{\algorithmicindent}\textcolor{gray}{\(\triangleright\) \textit{Stage 3 Dual-Group Graph Advantage (Section 3.4)}}
    \For{each semantic state $u$ with $|\mathcal{C}_g(u)|>1$}
        \State $A_{\mathrm{cor}}(e)\leftarrow (r_{\mathrm{step}}(e)-\operatorname{mean}(\{r_{\mathrm{step}}(\hat{e})\mid \hat{e}\in\mathcal{E}_{\mathrm{cor}}(u)\}))/(\operatorname{std}(\{r_{\mathrm{step}}(\hat{e})\mid \hat{e}\in\mathcal{E}_{\mathrm{cor}}(u)\}))$ for each $e\in\mathcal{E}_{\mathrm{cor}}(u)$
    \EndFor
    \For{each equivalence class $Q$ with $|Q|\ge 2$}
        \State Find deepest common generation predecessor $a_Q$; compute $\ell_v=d(v)-d(a_Q)$, $\bar{S}_Q=\operatorname{mean}(\{S(v)\mid v\in Q\})$
        \State $A_{\mathrm{eff}}(v)\leftarrow \bar{S}_Q\cdot(-\ell_v-\operatorname{mean}(\{-\ell_q\mid q\in Q\}))/(\operatorname{std}(\{-\ell_q\mid q\in Q\}))$; distribute uniformly to generated transitions on the route to $v$
    \EndFor
    \Statex \hspace{\algorithmicindent}\textcolor{gray}{\(\triangleright\) \textit{Stage 4 Policy Update}}
    \State Update $\theta$ by minimizing $\mathcal{L}_{\mathrm{ppo}}(A_{\mathrm{cor}},m_{\mathrm{resp}})+\lambda_{\mathrm{eff}}\mathcal{L}_{\mathrm{ppo}}(A_{\mathrm{eff}},m_{\mathrm{eff}})+\beta_{\mathrm{KL}}\mathcal{L}_{\mathrm{KL}}-c_H\mathcal{H}(\pi_\theta)$
\EndFor
\end{algorithmic}
\end{algorithm}

\section{Detailed Description of Baselines}
\label{app:baseline}
\textbf{GRPO~\citep{guo2025deepseek}:}
GRPO aims to remove the learned critic from PPO-style RLVR. It samples multiple complete responses for each prompt and normalizes verifier rewards within the group to obtain response-level advantages.

\textbf{DAPO~\citep{yu2025dapo}:}
DAPO aims to stabilize and scale GRPO for long-CoT reasoning. It keeps group-relative advantage estimation and improves training with decoupled clipping, dynamic sampling, token-level policy-gradient normalization, and overlong reward shaping.

\textbf{TreeRL~\citep{hou2025treerl}:}
TreeRL aims to improve exploration and credit assignment beyond independent response sampling. It builds an entropy-guided on-policy reasoning tree and estimates step values from the correctness ratio of descendant leaf responses.

\textbf{SPO~\citep{guosegment}:}
SPO aims to reduce the sparsity of response-level rewards. It partitions responses into segments, estimates segment-level advantages with Monte Carlo rollouts, and optimizes tokens using a probability-mask strategy.

\textbf{TREE-GRPO~\citep{ji2025tree}:}
TREE-GRPO aims to reuse shared prefixes and improve advantage estimation in tree-search RL. It was originally designed for agent trajectories; we adapt it to LLM reasoning by replacing agent interaction steps with reasoning segments while keeping its intra-tree and inter-tree grouped relative advantages.

\textbf{PROS~\citep{huang2026pros}:}
PROS aims to reduce redundant generation in RLVR by reusing rollout prefixes. It selects valuable historical prefixes, appends them to the original query as augmented prompts, and trains only on newly generated continuations.

\textbf{TreePO~\citep{yizhitreepo}:}
TreePO aims to improve rollout efficiency with tree modeling. It decodes responses in fixed-length segments, reuses shared prefixes through dynamic tree sampling, prunes low-value paths, and estimates segment advantages from global and local tree subgroups.

\section{Detailed Experimental Setup}
\label{app:exp_setup}
% 我们的GraphPO基于Verl~\citep{sheng2024hybridflow}框架搭建。在rollout阶段，使用vLLM框架进行高效的批量推理~\citep{kwon2023efficient}。我们的实验基于16*H20 GPUs进行。
\textbf{Framework and Hardware.}
For reasoning tasks, Our GraphPO is built on the Verl framework~\citep{sheng2024hybridflow}. During the rollout stage, we use the vLLM framework for efficient batched inference~\citep{kwon2023efficient}. 
In agentic RL, Our implementation is built upon Search-R1~\citep{jin2025search}.
Our experiments are conducted on 16$\times$H20 GPUs.

% 至于训练数据，我们使用DAPO-Math dataset对模型进行强化学习。在rollout过程中，我们让Chain方法（即GRPO、DAPO）每次rollout64条轨迹。对于Tree-based方法（即TreeRL、SPO、TREE-GRPO、PROS、TreePO）和GraphPO，我们固定每一个step的长度，然后使用和Chain方法接近的token预算进行rollout。具体计算方式为：
\textbf{Training Data.}
For training data in reasoning tasks, we use the DAPO-Math dataset~\citep{yu2025dapo} for reinforcement learning.
For agent tasks, we adopt the experimental setup of Tree-GRPO~\citep{ji2025tree}, using ASearcher-35K~\citep{gao2025beyond} and WebDancer~\citep{wu2025webdancer} as training data.

\textbf{Rollout Budget.}
In reasoning tasks, chain methods, including GRPO and DAPO, generate 64 trajectories for each prompt.
In deep search task, we use group size 8 for all chain-based RL

For tree-based methods, including TreeRL, SPO, TREE-GRPO, PROS, and TreePO, as well as GraphPO, we fix the length of each step and use a token budget close to that of chain methods for rollout. 
Taking a ternary tree as an example, the calculation is as follows.
Assume that the tree has $n$ layers and each branch step contains $s$ tokens.
Then each complete root-to-leaf trajectory has length $ns$, and the token budget of chain sampling with 64 trajectories is $B_{\mathrm{chain}} = 64ns.$
For the ternary tree, the number of generated branches at layer $i$ is $3^i$.
Therefore, the total rollout cost of constructing the tree is $B_{\mathrm{tree}}=s\sum_{i=1}^{n}3^i.$

% 我们设置DeepSeek-R1-Distill-Qwen-7B和 Qwen3-8BBase, we set a context length of 16384 tokens. For Qwen2.5-Math-7B, we use its maximum supported length of 4,096 tokens. Tree-based方法和GraphPO每一步大小我们设置为512到2048.
\textbf{Context Length and Step Size.}
For reasoning tasks, we set the context length to 16,384 tokens for DeepSeek-R1-Distill-Qwen-7B~\citep{guo2025deepseek} and Qwen3-8B-Base~\citep{yang2025qwen3}. For Qwen2.5-Math-7B~\citep{yang2024qwen2}, we use its maximum supported context length of 4,096 tokens. For tree based methods and GraphPO, we set the length of each step from 512 to 2048 tokens.
For deep search task, The max response length of Qwen2.5-7B-Instruct~\citep{qwen2025qwen25technicalreport} is set to 8000 tokens.

% 我们使用off policy RL设置进行训练。每个batch size取512，minibatch size为32。正文中的每个step为一个batch size。所有的模型我们都训练350step，使用AdamW优化器，学习率为1e-6，权重衰减为0.01。对于PPO的clip参数， The learning rate is set to 1×10−6 , with AdamW betas (0.9,0.999), weight decay 0.01, no warmup, PPO clip ratio 0.2, gradient clipping 1.0,(0.9,0.999), weight decay 0.01, no warmup, PPO clip ratio 0.2, gradient clipping 1.0。Actor 模块通过完全分片数据并行 (FSDP) 进行有效训练

\textbf{Optimization.}
We train the models under an off policy RL setting. For reasoning tasks, the batch size is set to 512, and the minibatch size is set to 32. In the main text, each training step corresponds to one batch. We train all models for 350 steps. 
For deep search task, we train 40 steps. The batch size is set to 128, and the minibatch size is set to 64.
We use the AdamW optimizer with a learning rate of $1\times10^{-6}$, betas of $(0.9, 0.999)$, and a weight decay of 0.01. We use no learning rate warmup. For PPO, we set the clip ratio to 0.2 and the gradient clipping threshold to 1.0. The actor module is trained efficiently with Fully Sharded Data Parallel.

\begin{table*}[t]
\centering
\caption{Performance comparison of different branching schemes on various reasoning benchmarks with {Qwen2.5-7B-Math}. We report Accuracy (\%) for each benchmark. The best performance for each benchmark is highlighted in \textbf{bold}.}
\resizebox{\textwidth}{!}{
\begin{tabular}{l|ccccc|c}
\toprule
\textbf{Method} & \textbf{AIME24} & \textbf{AIME25} & \textbf{MATH500} & \textbf{GPQA} & \textbf{LiveCodeBench} & \textbf{Average} \\
\midrule

Base & 13.6 & 6.3 & 54.8 & 28.5 & 5.7 & 21.8 \\
GRPO & 23.3 & 15.5 & 79.5 & 31.8 & 11.4 & 32.3 \\
DAPO & 25.7 & 17.2 & 82.6 & 32.3 & 12.6 & 34.1 \\

GraphPO ($\mathcal{T}=\{4,5,3,2\}$)       & 29.4 & 21.4 & 87.8 & 36.5 & 15.1 & 38.0 \\
GraphPO ($\mathcal{T}=\{7,6,5,4,3\}$)     & 31.0 & 23.2 & 90.0 & 38.3 & 16.1 & 39.7 \\
GraphPO ($\mathcal{T}=\{5,4,2,2,2\}$)     & 31.6 & 24.0 & 90.9 & 39.1 & 16.6 & 40.4 \\
GraphPO ($\mathcal{T}=\{2,4,3,2,2,2\}$)   & 30.2 & 22.3 & 89.0 & 37.4 & 15.6 & 38.9 \\
GraphPO ($\mathcal{T}=\{2,5,4,2,2\}$)   & \textbf{32.1} & \textbf{24.5} & \textbf{91.6} & \textbf{39.6} & \textbf{16.9} & \textbf{40.9} \\
\bottomrule
\end{tabular}
}
\label{tab:branch}
\end{table*}

\section{GraphPO Remains Effective Across Diverse Branching Strategies}
% 不同分支请情况下性能变化。表
% 构建RL Graph时每层每个节点生成多少个节点是非常关键的参数。根据Figure~\citep{fig:emperical}a可以发现，在推理刚开始时候不同的推理路径之间往往存在着较大的相似度，而在推理后期，各条推理路径差异逐渐变大。所以在实验过程中，我们采用的是逐层降低分支数量的分支方法，In this section，我们分析不同的设置下模型的性能。具体结果可见表~\ref{tab:branch}. 可以看见，逐层降低分支数量的分支方法能够取得更好的性能。这是因为在推理初期，较多的分支能够促进更多样化的探索，而在推理后期，较少的分支能够减少冗余探索，提高推理效率。同时，所有的分支方法都优于独立采样的方法，这说明在推理过程中引入分支能够促进更多样化的探索，从而提升推理性能。并且我们可以看出并不是分支数量越多越好，在token 预算相同的情况下，过多的分支会导致探索深度降低，从而思考时间缩短，无法找到更好的路径。
The number of steps generated by each reasoning path at each layer is a critical parameter when constructing the RL graph. According to Figure~\ref{fig:empirical}a and b, different reasoning paths tend to show high similarity at the early stage of reasoning. In the later stage, these reasoning paths gradually become more different. Therefore, in our experiments, we adopt a branching strategy that gradually reduces the number of branches across layers.

In this section, we analyze model performance under different settings. The results are shown in Table~\ref{tab:branch}. Here, $\mathcal{T}=\{b_1,b_2,\dots,b_H\}$ denotes the per-layer branching schedule of the reasoning graph: $H$ is the total number of layers, and $b_h$ specifies the number of child continuations sampled from each active state at layer $h$. For example, $\mathcal{T}=\{4,5,3,2\}$ corresponds to a $4$-layer graph in which each root state expands $4$ children at the first layer, every layer-$1$ state expands $5$ children at the second layer, and so on, with the branching factor decreasing toward deeper layers. We can see that the strategy with gradually decreasing branches achieves better performance. This is because more branches in the early reasoning stage can encourage more diverse exploration. In contrast, fewer branches in the later reasoning stage can reduce redundant exploration and improve reasoning efficiency.

At the same time, all branching strategies outperform independent sampling. This shows that introducing branches during reasoning can encourage more diverse exploration and improve reasoning performance. We can also see that using more branches is not always better. Under the same token budget, too many branches reduce the exploration depth. As a result, the model has less reasoning time and may fail to find better reasoning paths. Moreover, the comparison between $\mathcal{T}=\{4,5,3,2\}$ and the uniform schedule $\mathcal{T}=\{4,4,4,4\}$ further confirms that allocating more capacity to mid-stage branching, where reasoning paths begin to diverge, is more beneficial than spreading branches uniformly across all layers. This observation aligns with our empirical study in Section~\ref{sec:Empiricalstudy}, where the divergence of reasoning trajectories peaks in the middle stage rather than at the very beginning or the very end.

\section{Ablation Study}
\label{app:ablation}
% 去掉efficiency advantage的回复长度，正确率小图，
\begin{wrapfigure}{r}{0.35\textwidth}  % r表示图片在右边；l为左边
  \centering
  \vspace{-10pt}
  \includegraphics[width=0.35\textwidth]{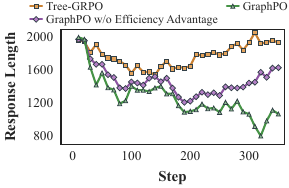}
    \vspace{-13pt}
  \caption{Impact of efficiency advantage on response length.}
  \label{fig:ablation_length}
    \vspace{-10pt}
\end{wrapfigure}
\textbf{Efficiency advantage improves reasoning efficiency.}
% 为了证明加入efficiency advantage后能够提升模型的推理效率，我们进一步分析了efficiency advantage对模型回复长度的影响。实验结果如图~\ref{fig:ablation_length}所示。我们可以看出不使用efficiency advantage后，在训练过程中模型的回复长度和Tree-based 方法一样先下降后升高。这主要是因为更长的回复会带来更多的分支机会，从而让模型学会冗余探索的pattern，降低推理效率。同时我们也可以发现即使不使用efficiency advantage，GraphPO的回复长度总体仍旧会比tree更低。因为在融合节点时会降低长推理路径的采样数量，导致学习到的后续信息很多来自于更短路径的后续推理，所以最后学习得到的更倾向于学习短路径的回答模式。
To show that the efficiency advantage improves reasoning efficiency, we further analyze its effect on response length. The results are shown in Figure~\ref{fig:ablation_length}. 
We can see that, without the efficiency advantage, the response length first decreases and then increases during training, similar to tree based methods. This is mainly because longer responses create more branching opportunities. As a result, the model may learn redundant exploration patterns, which reduces reasoning efficiency.
We can also observe that even without the efficiency advantage, GraphPO still produces shorter responses than tree based methods overall. This is because node merging reduces the sampling budget of longer reasoning paths. Therefore, much of the learned downstream information comes from shorter paths. As a result, the final model tends to learn response patterns based on shorter reasoning paths.

% 共享w值的影响
% 去掉efficiency advantage的回复长度，正确率小图，
\begin{wrapfigure}{r}{0.35\textwidth}  % r表示图片在右边；l为左边
  \centering
  \vspace{-10pt}
  \includegraphics[width=0.35\textwidth]{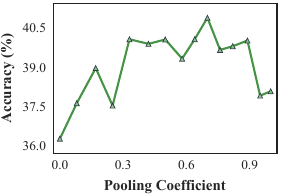}
    \vspace{-13pt}
  \caption{Impact of pooling coefficient.}
  \label{fig:ablation_w}
    \vspace{-10pt}
\end{wrapfigure}
\textbf{Moderate Pooling Coefficients Strike the Best Bias-Variance Balance.} The pooling coefficient $w\in[0,1]$ controls how strongly the suffixes shared from equivalent nodes contribute to a node's value estimate. In this section, we conduct ablation studies to analyze the impact of the pooling coefficient. Figure~\ref{fig:ablation_w} shows the performance of GraphPO under different pooling coefficients. We can see that the performance first increases and then decreases as $w$ becomes larger. This is because a small pooling coefficient prevents the graph from sharing suffixes across equivalent states, so the advantage estimator inherits the high variance of tree-based estimators with few terminal samples per node and yields noisy gradients that slow convergence. In contrast, a large pooling coefficient forces the value estimate to rely heavily on suffixes from semantically similar but not identical nodes. The residual semantic heterogeneity is amplified into a non-negligible merging bias, which distorts the process reward and hurts policy optimization. When $w=0$, the equivalence-class suffix sharing is disabled and GraphPO degenerates to a tree-based estimator, and the performance drops clearly. A moderate value around $w=0.7$ best balances the variance reduction obtained from shared suffixes and the semantic-merging bias, which is consistent with the theoretical analysis in Lemma~\ref{lem:graphvar}.

% 不同分支请情况下性能变化。表
\section{Theoretical Analysis}\label{app:theory}

In this section we provide formal theoretical justification for the four claims made in the main text:
(i) compared with tree-based methods that cannot reuse nodes, equivalence-class suffix sharing pools downstream evidence from all semantically equivalent paths into a single advantage-estimation group, which yields lower-variance node-score and advantage estimates at the cost of a controllable $O(1{-}\kappa)$ semantic-merging bias (Section~\ref{app:variance});
(ii) GraphPO improves expected exploration efficiency under a fixed token budget when budget saved
from redundant continuations is redirected to novel frontiers (Section~\ref{app:exploration});
(iii) in expectation, the GraphPO surrogate trained with $r_{\mathrm{step}}$ has the same first-order update direction as a novelty-gated oracle PRM surrogate, and thus yields self-emergent process supervision purely from outcome rewards (Section~\ref{app:prm});
and (iv) the dual-group graph advantage breaks outcome-objective ties in favor of shorter
semantically equivalent correct reasoning paths (Section~\ref{app:inference}).

\subsection{Notation and Assumptions}\label{app:notation}
Let $x$ denote a prompt, $\pi_\theta$ the policy, and $\tau=(s_0,e_1,s_1,\ldots,e_T,s_T)$ a reasoning
trajectory with reasoning steps $e_t$ as edges and semantic states $s_t$ as nodes.
Let $V^\star(s)\!=\!\mathbb{E}_{\tau\sim\pi_\theta\mid s_0=s}[R(\tau)]$ be the
true value function under the verifier reward $R\in\{0,1\}$.
Let $\phi:\Sigma^\ast\!\to\!\mathbb{R}^d$ denote the embedder, with semantic similarity
$\operatorname{sim}(u,v)=\langle \phi(u),\phi(v)\rangle/(\|\phi(u)\|\|\phi(v)\|)$ and
$\kappa\in(0,1)$ the merging threshold.

\begin{assumption}[Value-stable detected equivalence classes]\label{ass:lip}
There exists a class-stability radius $\delta_\kappa=O(1-\kappa)$ such that for any two
states $u,v$ in the same detected equivalence class,
\begin{equation}
|V^\star(u)-V^\star(v)|\le \delta_\kappa .
\end{equation}
This is the class-level version of a Lipschitz semantic-equivalence condition. It avoids
treating a union-find transitive closure as automatically pairwise close unless the detected
class itself has small semantic diameter.
\end{assumption}

\begin{assumption}[Bounded rewards and i.i.d.\ rollouts inside a group]\label{ass:bounded}
Verifier rewards are bounded in $[0,1]$. Conditional on a parent state $u$, the children
generated by $\pi_\theta(\cdot\mid u)$ are i.i.d., and their downstream completions are
independent given the children. The own-subtree terminal samples used in a pooled score are
sampled independently and are disjoint across equivalence-class members.
\end{assumption}

\begin{assumption}[No-cycle merging]\label{ass:nocycle}
Equivalence classes $\mathcal{Q}(v)$ are formed only between non-causal pairs, so the merged
structure remains a DAG and the bottom-up propagation in Section~\ref{sec:graph_reward}
is well-defined.
\end{assumption}

These assumptions are consistent with standard analysis of RLVR/GRPO-style estimators
\citep{kazemnejad2024vineppo,setlurrewarding,cui2025process} and with the structural
constraints of GraphPO described in Section~\ref{sec:graph_construction}.
Throughout the analysis, variances are conditional on the sampled graph structure and on the
observed terminal rollout counts.

\subsection{Lower-Variance Advantage Estimation via Shared Suffixes}\label{app:variance}

We first analyze the estimation variance of the advantage estimators used by GRPO, TreeRPO, and GraphPO. Trees expand semantically equivalent states as independent nodes, so each state's score is estimated only from its own subtree. GraphPO instead merges these states into an equivalence class and pools their terminal samples, which enlarges the effective sample size of each node-score estimate and of the standardization group used to form the advantage. Both effects reduce the variance of the resulting advantage estimate.

\paragraph{Estimators.}
For a state $u$ with $K$ children, let $\hat V_i$ be the empirical correctness rate of the
$i$-th child obtained from $m$ terminal rollouts. Define:
\begin{equation}
\begin{aligned}
\hat A^{\mathrm{GRPO}}_i &= \hat R_i - \tfrac{1}{K}\textstyle\sum_{j=1}^K \hat R_j,\\
\hat A^{\mathrm{Tree}}_i &= \hat V_i - \tfrac{1}{K}\textstyle\sum_{j=1}^K \hat V_j,\\
\hat A^{\mathrm{Graph}}_i &= \hat V^{Q}_i - \tfrac{1}{K}\textstyle\sum_{j=1}^K \hat V^{Q}_j,\quad
\hat V^{Q}_i = \frac{c_i^{\mathrm{own}}+w\!\sum_{q\in\mathcal{Q}(i)\setminus\{i\}}c_q^{\mathrm{own}}}
{t_i^{\mathrm{own}}+w\!\sum_{q\in\mathcal{Q}(i)\setminus\{i\}}t_q^{\mathrm{own}}}.
\end{aligned}
\end{equation}
GRPO uses raw outcome rewards on full trajectories.
TreeRPO uses Monte-Carlo correctness of each child computed only from its own subtree.
GraphPO additionally pools terminal evidence across the equivalence class $\mathcal{Q}(i)$.
These are the raw pre-normalization estimators; the standardization in
Section~\ref{sec:graph_advantage} rescales the signal after the sampling estimate is formed.

\begin{lemma}[Per-child variance]\label{lem:perchild}
Under Assumptions~\ref{ass:lip}--\ref{ass:bounded}, with $m$ Monte-Carlo terminal rollouts per child,
\begin{equation}
\operatorname{Var}\!\bigl(\hat V_i\bigr)\;=\;\frac{V^\star(i)(1-V^\star(i))}{m}.
\end{equation}
\end{lemma}
\begin{proof}
$\hat V_i$ is the empirical mean of $m$ i.i.d.\ Bernoulli($V^\star(i)$) terminal indicators by
Assumption~\ref{ass:bounded}. The variance of a Bernoulli mean is $p(1-p)/m$.
\end{proof}

\begin{lemma}[Variance with equivalence-class pooling]\label{lem:graphvar}
Let $|\mathcal{Q}(i)|=k_i\!\ge\!1$ and suppose each member of $\mathcal{Q}(i)$ has $m$
own terminal rollouts. Define
\begin{equation}
\rho_i=\frac{1+(k_i-1)w^2}{(1+w(k_i-1))^2}\le \frac{1}{1+w(k_i-1)} .
\end{equation}
Then
\begin{equation}
\operatorname{Var}(\hat V^{Q}_i)
\;\le\;
\rho_i\,\frac{\max_{q\in\mathcal{Q}(i)}V^\star(q)(1-V^\star(q))}{m},
\end{equation}
and, since $p(1-p)$ is 1-Lipschitz on $[0,1]$,
\begin{equation}
\operatorname{Var}(\hat V^{Q}_i)
\;\le\;
\rho_i\,\frac{V^\star(i)(1-V^\star(i))+\delta_\kappa}{m},
\end{equation}
and its semantic-merging bias satisfies
\begin{equation}
\bigl|\mathbb{E}[\hat V^{Q}_i]-V^\star(i)\bigr|\le \delta_\kappa .
\end{equation}
Consequently, the mean-squared error relative to $V^\star(i)$ is bounded by the variance
term above plus $\delta_\kappa^2$.
\end{lemma}
\begin{proof}
For balanced groups, write
$\hat V^{Q}_i=\sum_{q\in\mathcal{Q}(i)}\alpha_q\hat V_q$, where
$\alpha_i=1/(1+w(k_i-1))$, $\alpha_q=w/(1+w(k_i-1))$ for $q\ne i$, and
$\sum_q\alpha_q=1$. By Assumption~\ref{ass:bounded}, the $\hat V_q$ are mutually independent,
so
\begin{equation}
\operatorname{Var}(\hat V^{Q}_i)
=\sum_{q\in\mathcal{Q}(i)}\alpha_q^2\operatorname{Var}(\hat V_q).
\end{equation}
Using Lemma~\ref{lem:perchild} gives the variance bound because
$\sum_q\alpha_q^2=\rho_i$. The second inequality in the definition of $\rho_i$ follows from
$w^2\le w$ for $w\in[0,1]$. The alternative variance bound follows from
Assumption~\ref{ass:lip} and the fact that $p(1-p)$ is 1-Lipschitz on $[0,1]$. For the bias,
\begin{equation}
\left|\mathbb{E}[\hat V^{Q}_i]-V^\star(i)\right|
\le \sum_{q\in\mathcal{Q}(i)}\alpha_q|V^\star(q)-V^\star(i)|
\le \delta_\kappa
\end{equation}
by Assumption~\ref{ass:lip}.
\end{proof}

\begin{theorem}[GraphPO yields lower-variance advantage estimates via shared suffixes]\label{thm:var}
Under Assumptions~\ref{ass:lip}--\ref{ass:nocycle}, suppose the $K$ child-score estimates
used in the group baseline are conditionally independent and have balanced own terminal
rollouts. Let $\rho_{\max}(u)=\max_{i\le K}\rho_i$. Then
\begin{equation}
\begin{aligned}
\operatorname{Var}\!\bigl(\hat A^{\mathrm{Graph}}_i\bigr)
\;&\le\;\rho_{\max}(u)\,
\operatorname{Var}\!\bigl(\hat A^{\mathrm{Tree}}_i\bigr)
+O\!\left(\frac{\rho_{\max}(u)\delta_\kappa}{m}\right)\\
\;&\le\;\rho_{\max}(u)\,
\operatorname{Var}\!\bigl(\hat A^{\mathrm{GRPO}}_i\bigr)
+O\!\left(\frac{\rho_{\max}(u)\delta_\kappa}{m}\right),
\end{aligned}
\end{equation}
where the second inequality compares to a GRPO estimator that uses a single terminal outcome
per child. The estimator's bias relative to the oracle advantage is $O(\delta_\kappa)$, and
its MSE has an additional $O(\delta_\kappa^2)$ term. If every child in the comparison group
has $k_i>1$ and $w>0$, then $\rho_{\max}(u)<1$, so the variance reduction from suffix sharing is
strict whenever the semantic-heterogeneity term is smaller than the saved sampling variance, i.e.\ GraphPO returns a lower-variance advantage estimate than tree-based baselines that cannot reuse nodes.
\end{theorem}
\begin{proof}
For independent child-score estimates $X_1,\ldots,X_K$,
\begin{equation}
\operatorname{Var}\!\left(X_i-\frac{1}{K}\sum_{j=1}^K X_j\right)
=\left(1-\frac{1}{K}\right)^2\operatorname{Var}(X_i)
+\frac{1}{K^2}\sum_{j\ne i}\operatorname{Var}(X_j).
\end{equation}
All coefficients are nonnegative. Lemma~\ref{lem:graphvar} contracts each child-score
variance by at most $\rho_{\max}(u)$ relative to the non-pooled Tree estimator, up to the
$O(\rho_{\max}(u)\delta_\kappa/m)$ heterogeneity term induced by imperfect semantic merging.
The same bound therefore applies after subtracting the group baseline. A Tree estimator averages
$m$ terminal Bernoulli outcomes per child, while the corresponding GRPO estimator uses one
terminal outcome per child; the law of total variance therefore gives
$\operatorname{Var}(\hat A^{\mathrm{Tree}}_i)\le
\operatorname{Var}(\hat A^{\mathrm{GRPO}}_i)$. The bias and MSE statements follow from the
bias part of Lemma~\ref{lem:graphvar}.
\end{proof}

\begin{corollary}[More accurate edge rewards via aggregation over equivalent partners]\label{cor:processvar}
The step-level edge reward $r_{\mathrm{step}}(u,v)=(S(v)-S(u))(1-\eta(u,v))$ inherits the
estimation accuracy of its endpoint scores. Define the oracle edge reward
$r^\star_{\mathrm{step}}(u,v)=(V^\star(v)-V^\star(u))(1-\eta(u,v))$ and the empirical
plug-in $\hat r^{\mathrm{Graph}}_{\mathrm{step}}(u,v)$ obtained by replacing $S$ with the
pooled estimator $\hat V^{Q}$ from Lemma~\ref{lem:graphvar}, and analogously $\hat r^{\mathrm{Tree}}_{\mathrm{step}}$ from the non-pooled tree estimator. Let
$\Gamma^{\mathrm{Graph}}_{uv}=\operatorname{Cov}(\hat V^{Q}(u),\hat V^{Q}(v))$ and
$\Gamma^{\mathrm{Tree}}_{uv}=\operatorname{Cov}(\hat V^{\mathrm{Tree}}(u),\hat V^{\mathrm{Tree}}(v))$.
Then the mean-squared error of the GraphPO edge reward to the oracle satisfies
\begin{equation}
\begin{aligned}
\operatorname{MSE}\bigl(\hat r^{\mathrm{Graph}}_{\mathrm{step}}(u,v)\bigr)
\le{}&(1-\eta(u,v))^2
\Bigl(\rho_u\operatorname{Var}(\hat V^{\mathrm{Tree}}(u))
+\rho_v\operatorname{Var}(\hat V^{\mathrm{Tree}}(v))
-2\Gamma^{\mathrm{Graph}}_{uv}\Bigr)\\
&+\,(1-\eta(u,v))^2\,(2\delta_\kappa)^2
+O\!\left(\frac{(\rho_u+\rho_v)\delta_\kappa}{m}\right).
\end{aligned}
\end{equation}
Here $\rho_u,\rho_v\in(0,1]$ are the endpoint pooling factors from Lemma~\ref{lem:graphvar},
the first line bounds the variance contributed by the two endpoint estimators (which is
strictly reduced by aggregating own-subtree evidence with that of equivalent partners),
and the $(2\delta_\kappa)^2$ term bounds the squared bias induced by imperfect semantic
merging. Compared with the tree edge reward, which uses no information from equivalent
partners, the GraphPO edge reward is strictly more accurate (in MSE) whenever the variance
saving from shared evidence dominates the residual semantic-heterogeneity bias:
\begin{equation}
\begin{aligned}
&(1-\eta(u,v))^2
\Bigl((1-\rho_u)\operatorname{Var}(\hat V^{\mathrm{Tree}}(u))
+(1-\rho_v)\operatorname{Var}(\hat V^{\mathrm{Tree}}(v))
-2(\Gamma^{\mathrm{Tree}}_{uv}-\Gamma^{\mathrm{Graph}}_{uv})\Bigr)\\
&\quad>\;
(1-\eta(u,v))^2(2\delta_\kappa)^2
+O\!\left(\frac{(\rho_u+\rho_v)\delta_\kappa}{m}\right).
\end{aligned}
\end{equation}
In the conditionally independent endpoint case the covariance term vanishes and the
condition reduces to nontrivial pooling ($w>0$, $|\mathcal{Q}|>1$) with small semantic
heterogeneity ($\delta_\kappa$ small).
\end{corollary}
\begin{proof}
Decompose MSE into variance and squared bias. The reward is a deterministic linear
transformation of the two endpoint score estimates:
$\hat r_{\mathrm{step}}-r^\star_{\mathrm{step}}=(1-\eta)\bigl((\hat S(v)-V^\star(v))-(\hat S(u)-V^\star(u))\bigr)$,
hence
\begin{equation}
\begin{aligned}
\operatorname{MSE}(\hat r_{\mathrm{step}})
={}&(1-\eta)^2\Bigl(\operatorname{Var}(\hat S(v))+\operatorname{Var}(\hat S(u))
-2\operatorname{Cov}(\hat S(u),\hat S(v))\Bigr)\\
&+(1-\eta)^2\bigl(\mathbb{E}\hat S(v)-V^\star(v)-(\mathbb{E}\hat S(u)-V^\star(u))\bigr)^2.
\end{aligned}
\end{equation}
For the GraphPO estimator, applying the variance bound of Lemma~\ref{lem:graphvar} to each
endpoint variance gives the first line of the bound; applying the bias bound
$|\mathbb{E}\hat V^{Q}(\cdot)-V^\star(\cdot)|\le\delta_\kappa$ at both endpoints and the
triangle inequality gives the $(2\delta_\kappa)^2$ squared-bias term. The tree estimator is
unbiased (zero bias term) but cannot pool over equivalent partners, so its variance term uses
$\rho=1$. Subtracting the two MSE expressions yields the stated comparison condition: GraphPO
trades a small $O(\delta_\kappa^2)$ bias for an $\Omega(1-\rho)$ variance reduction, and is
strictly more accurate whenever the latter dominates.
\end{proof}

\paragraph{Dual-group variance reduction.}
The dual-group graph advantage further benefits from equivalence-class pooling because the
standardization group itself is enlarged. In a tree estimator that cannot reuse nodes, the
group used to standardize the correctness advantage at state $u$ contains only the direct
children $\mathcal{C}(u)$. In GraphPO, paths reaching equivalent states are merged, so the
group becomes $\mathcal{E}_{\mathrm{cor}}(u)=\{(p(v),v)\mid v\in\mathcal{C}_g(u)\}$ with
$|\mathcal{C}_g(u)|=\sum_{q\in\mathcal{Q}(u)}|\mathcal{C}(q)|$.

\begin{proposition}[Lower correctness-advantage variance via equivalence-class grouping]\label{prop:groupvar}
Fix a source state $u$ and let $K_T=|\mathcal{C}(u)|$ and $K_G=|\mathcal{C}_g(u)|\ge K_T$ be
the tree- and graph-group sizes. Assume the per-edge step-reward estimates
$\{\hat r_{\mathrm{step}}(\hat e)\}$ are conditionally independent within the group with a
common variance bound $\sigma^2_r$. Let $\hat A^{\mathrm{Tree}}_{\mathrm{cor}}$ and
$\hat A^{\mathrm{Graph}}_{\mathrm{cor}}$ be the pre-standardization centered advantages
$\hat r_{\mathrm{step}}(e)-\operatorname{mean}(\cdot)$ computed on the tree and graph
groups, respectively. Then
\begin{equation}
\operatorname{Var}\!\bigl(\hat A^{\mathrm{Graph}}_{\mathrm{cor}}\bigr)
\;\le\;\Bigl(1-\tfrac{1}{K_G}\Bigr)\sigma^2_r
\;\le\;\Bigl(1-\tfrac{1}{K_T}\Bigr)\sigma^2_r
=\operatorname{Var}\!\bigl(\hat A^{\mathrm{Tree}}_{\mathrm{cor}}\bigr),
\end{equation}
with strict inequality whenever $K_G>K_T$ (i.e.\ at least one non-trivial equivalence-class
member contributes shared descendants).
\end{proposition}
\begin{proof}
For independent variates $X_1,\dots,X_K$ with common variance $\sigma^2$,
\begin{equation}
\operatorname{Var}\!\bigl(X_i-\tfrac{1}{K}\textstyle\sum_{j=1}^K X_j\bigr)
=\bigl(1-\tfrac{1}{K}\bigr)^2\sigma^2+\tfrac{K-1}{K^2}\sigma^2
=\bigl(1-\tfrac{1}{K}\bigr)\sigma^2,
\end{equation}
which decreases monotonically in $K$. Applying this with $K\!\in\!\{K_T,K_G\}$ and using
$K_G\ge K_T$ yields the inequality.
\end{proof}

Proposition~\ref{prop:groupvar} formalizes the dual-group claim of
Section~\ref{sec:graph_advantage}: by aggregating semantically equivalent paths into one
correctness-advantage group, GraphPO obtains a more concentrated baseline and standard
deviation than tree-based estimators that cannot reuse nodes, which directly lowers the
variance of $A_{\mathrm{cor}}$.

\subsection{Improved Exploration Efficiency}\label{app:exploration}

We now formalize the exploration-efficiency claim of Section~\ref{sec:Empiricalstudy}.
Let the rollout budget be $B$ tokens and let $\mathcal S_B$ be the sampled reasoning steps.
Define the cumulative useful exploration as
\begin{equation}
G(B)=\sum_{s\in\mathcal S_B}|s|\mathbf{1}[s\text{ correct}],
\end{equation}
so $\bar{\mathrm{Eff}}(B)=\mathbb{E}[G(B)]/B$ is the expected per-budget version of
Eq.~\eqref{eq:eq1} when the consumed budget is $B$. Let $p_{\mathrm{red}}$ be the fraction
of rollout budget spent on states that are both redundant and detected as semantically
equivalent to a previously sampled state in the same layer.

\begin{proposition}[Budget reallocation gain]\label{prop:budget}
Suppose that, under tree-based sampling, a fraction $p_{\mathrm{red}}$ of the rollout budget is spent on frontier states that are identified as redundant by GraphPO's equivalence detector. Assume that the average marginal gain per unit budget on these redundant states is at most $\Delta_r$, while the average marginal gain per unit budget on novel frontier states receiving the reallocated budget is at least $\Delta_n$. Suppose also that GraphPO redirects the budget saved by its halved-budget rule to novel frontiers. If $\Delta_n>\Delta_r$, then, ignoring integer rounding,
\[
\bar{\mathrm{Eff}}_{\mathrm{Graph}}(B)-\bar{\mathrm{Eff}}_{\mathrm{Tree}}(B)
\ge
\frac{1}{2}p_{\mathrm{red}}(\Delta_n-\Delta_r)>0.
\]
\end{proposition}

\begin{proof}
Under tree-based sampling, a fraction $p_{\mathrm{red}}$ of the budget is spent on redundant frontier states. By assumption, the average marginal gain of this budget is at most $\Delta_r$. The remaining budget is spent on non-redundant frontier states.

GraphPO applies the halved-budget rule to detected redundant states. Therefore, ignoring integer rounding, it saves a budget fraction $\frac{1}{2}p_{\mathrm{red}}$ from redundant expansions. This saved budget is redirected to novel frontier states, whose average marginal gain is at least $\Delta_n$ by assumption.

Compared with tree-based sampling, GraphPO replaces a budget fraction $\frac{1}{2}p_{\mathrm{red}}$ whose marginal gain is at most $\Delta_r$ with the same budget fraction whose marginal gain is at least $\Delta_n$. Hence the expected gain per unit budget improves by at least
\[
\frac{1}{2}p_{\mathrm{red}}\Delta_n
-
\frac{1}{2}p_{\mathrm{red}}\Delta_r
=
\frac{1}{2}p_{\mathrm{red}}(\Delta_n-\Delta_r).
\]
The improvement is strictly positive whenever $p_{\mathrm{red}}>0$ and $\Delta_n>\Delta_r$.
\end{proof}

\begin{theorem}[Exploration coverage]\label{thm:coverage}
Let $N_{\mathrm{cov}}(B)$ be the number of \emph{distinct} semantic states discovered with
budget $B$. Assume an average expansion cost $\bar c$, and assume each expansion from a
novel frontier discovers a new semantic state with probability at least $\mu>0$. Under
the same detected-redundancy and saved-budget reallocation conditions as Proposition~\ref{prop:budget},
\begin{equation}
\mathbb{E}\!\left[N^{\mathrm{Graph}}_{\mathrm{cov}}(B)\right]
\;\ge\;
\mathbb{E}\!\left[N^{\mathrm{Tree}}_{\mathrm{cov}}(B)\right]
+\frac{\mu p_{\mathrm{red}}B}{2\bar c}.
\end{equation}
In the uniform-yield case where
$\mathbb{E}[N^{\mathrm{Tree}}_{\mathrm{cov}}(B)]=\mu(1-p_{\mathrm{red}})B/\bar c$ and
$p_{\mathrm{red}}<1$, this implies the multiplicative gain
\begin{equation}
\frac{\mathbb{E}[N^{\mathrm{Graph}}_{\mathrm{cov}}(B)]}
{\mathbb{E}[N^{\mathrm{Tree}}_{\mathrm{cov}}(B)]}
\;\ge\;
1+\frac{p_{\mathrm{red}}}{2(1-p_{\mathrm{red}})}.
\end{equation}
\end{theorem}
\begin{proof}
Each token a tree spends inside an already-discovered equivalence class produces no new state.
GraphPO detects such tokens via $\operatorname{sim}\!\ge\!\kappa$ and reallocates half of the
remaining budget to active novel frontiers. The reallocated budget is
$p_{\mathrm{red}}B/2$, which corresponds to $p_{\mathrm{red}}B/(2\bar c)$ additional
frontier expansions in expectation. Since each such expansion discovers a new state with
probability at least $\mu$, the additive bound follows. Dividing this bound by the uniform
tree coverage expression gives the stated multiplicative form.
\end{proof}

\subsection{Self-Emergent Process Supervision (PRM-like signal)}\label{app:prm}

We now show that the objective optimized with GraphPO's edge reward has the same
first-order policy-update direction as a novelty-gated oracle PRM surrogate, up to a
controllable $O(1-\kappa)$ semantic-merging bias. The novelty factor $1-\eta(u,v)$ is a
non-negative reweighting: it preserves the sign of the oracle advantage on informative
steps and attenuates credit on semantically redundant ones (where the oracle advantage is
itself near zero by Assumption~\ref{ass:lip}). Thus $r_{\mathrm{step}}$ is not aiming at a
different target than PRM; it is a sample-efficient, novelty-gated estimator of the same
per-edge improvement signal.

\begin{lemma}[Bias of node score]\label{lem:unbiased}
Under Assumptions~\ref{ass:lip}--\ref{ass:bounded},
$|\mathbb{E}[S(v)]-V^\star(v)|\le \delta_\kappa$
for every node $v$ with positive pooled terminal count.
\end{lemma}
\begin{proof}
Conditional on the rollout counts, $S(v)$ is a convex combination
$\sum_{q\in\mathcal{Q}(v)}\beta_q\hat V_q$ with weights
$\beta_q=w_q t_q^{\mathrm{own}}/(t_v^{\mathrm{own}}+w\sum_{q'\ne v}t_{q'}^{\mathrm{own}})$
($w_v=1$, $w_q=w$ for $q\ne v$), so $\sum_q\beta_q=1$.
$\mathbb{E}\hat V_q=V^\star(q)$ and Assumption~\ref{ass:lip} give
$|V^\star(q)-V^\star(v)|\le \delta_\kappa$. Hence
$|\mathbb{E}[S(v)]-V^\star(v)|\le \sum_q\beta_q|V^\star(q)-V^\star(v)|\le \delta_\kappa$.
\end{proof}

\begin{theorem}[GraphPO matches the novelty-gated PRM surrogate gradient]\label{thm:prm}
Let $A^\star_{\mathrm{step}}(u,v)=V^\star(v)-V^\star(u)$ denote the oracle process
advantage an idealized PRM would assign to the step $u\!\to\!v$, and let
$\omega(u,v)=1-\eta(u,v)\in[0,1]$ denote the novelty factor. Let
$\rho_\theta(u,v)=\pi_\theta(v\mid u)/\pi_{\theta_{\mathrm{old}}}(v\mid u)$ be the edge-level
shorthand for the token-level importance ratios in Section~\ref{sec_policy_objective}.
Define the novelty-gated oracle PRM
surrogate and the GraphPO step-reward surrogate as
\begin{equation}
\begin{aligned}
\mathcal{J}^{\omega\star}_{\mathrm{PRM}}(\theta)
&=\mathbb{E}_{\tau,\mathcal{G}\sim\pi_{\theta_{\mathrm{old}}}}\!\Big[\sum_{(u,v)\in\tau}\rho_\theta(u,v)\,\omega(u,v)A^\star_{\mathrm{step}}(u,v)\Big],\\
\mathcal{J}_{\mathrm{step}}(\theta)
&=\mathbb{E}_{\tau,\mathcal{G}\sim\pi_{\theta_{\mathrm{old}}}}\!\Big[\sum_{(u,v)\in\tau}\rho_\theta(u,v)\,r_{\mathrm{step}}(u,v)\Big].
\end{aligned}
\end{equation}
Here $\mathcal{G}$ determines $S(\cdot)$ and therefore $r_{\mathrm{step}}$, and all rewards
are detached during the policy update. Under Assumptions~\ref{ass:lip}--\ref{ass:bounded},
at $\theta=\theta_{\mathrm{old}}$,
\begin{equation}\label{eq:prm-grad-align}
\nabla_\theta \mathcal{J}_{\mathrm{step}}(\theta)\big|_{\theta_{\mathrm{old}}}
=\nabla_\theta \mathcal{J}^{\omega\star}_{\mathrm{PRM}}(\theta)\big|_{\theta_{\mathrm{old}}}
+O(\delta_\kappa)\,\mathbf{b}(\theta_{\mathrm{old}}),
\end{equation}
with $\|\mathbf{b}(\theta_{\mathrm{old}})\|\le \mathbb{E}_\tau\!\big[\sum_{(u,v)\in\tau}\|\nabla_\theta\log\pi_\theta(v\mid u)\|\big]$. Two consequences follow.
\emph{(a) Objective-level alignment.} The GraphPO objective trained with $r_{\mathrm{step}}$
has the same first-order update direction as the novelty-gated oracle PRM surrogate, up to
the $O(\delta_\kappa)$ merging bias. Compared with the ungated oracle PRM surrogate, each
edge contribution is multiplied by the non-negative scalar $\omega(u,v)$, so the update
remains sign-consistent with the oracle PRM while suppressing redundant steps.
\emph{(b) Novelty-gated denoising.} For semantically equivalent endpoints
($\eta\to 1$, $\omega\to 0$), Assumption~\ref{ass:lip} already bounds
$|A^\star_{\mathrm{step}}(u,v)|\le\delta_\kappa$, so $\omega$ attenuates edges on which
the oracle credit is itself $O(\delta_\kappa)$; for novel steps
($\eta\approx 0$, $\omega\approx 1$), the per-edge surrogate gradient coincides with the PRM
gradient up to $O(\delta_\kappa)$. Finally, with $m$ terminal rollouts per
descendant the MSE of $r_{\mathrm{step}}$ as an estimator of
$\omega(u,v)A^\star_{\mathrm{step}}(u,v)$ satisfies
$O((\rho_u+\rho_v)/m)+O(\delta_\kappa^2)$.
\end{theorem}
\begin{proof}
The proof follows the objective used in Section~\ref{sec_policy_objective}. In one PPO/DAPO
update, sampled graphs and their rewards are fixed, so differentiating the unclipped
surrogate gives
\begin{equation}
\nabla_\theta \mathcal{J}_{\mathrm{step}}(\theta)\big|_{\theta_{\mathrm{old}}}
=\mathbb{E}_{\tau,\mathcal{G}\sim\pi_{\theta_{\mathrm{old}}}}\!\left[\sum_{(u,v)\in\tau}
r_{\mathrm{step}}(u,v)\nabla_\theta\log\pi_\theta(v\mid u)\big|_{\theta_{\mathrm{old}}}\right],
\end{equation}
because $\rho_{\theta_{\mathrm{old}}}(u,v)=1$ and
$\nabla_\theta\rho_\theta(u,v)|_{\theta_{\mathrm{old}}}=\nabla_\theta\log\pi_\theta(v\mid u)|_{\theta_{\mathrm{old}}}$. Conditioning
on an edge's endpoints marginalizes over the downstream rollout graph used to compute
$S(u)$ and $S(v)$. Because $\omega(u,v)$ is a deterministic function of the endpoint embeddings and
$\mathbb{E}[S(w)\mid w]=V^\star(w)+O(\delta_\kappa)$ by Lemma~\ref{lem:unbiased},
\begin{equation}
\mathbb{E}\!\bigl[r_{\mathrm{step}}(u,v)\bigm|u,v\bigr]
=\omega(u,v)\bigl(\mathbb{E}[S(v)\mid v]-\mathbb{E}[S(u)\mid u]\bigr)
=\omega(u,v)\,A^\star_{\mathrm{step}}(u,v)+O(\delta_\kappa).
\end{equation}
Substituting this conditional expectation into the surrogate gradient gives
Eq.~\eqref{eq:prm-grad-align}; the bound on $\mathbf{b}(\theta_{\mathrm{old}})$
follows from the triangle inequality applied to the $O(\delta_\kappa)$ residual.
Claim (a) is immediate by comparing Eq.~\eqref{eq:prm-grad-align} with
$\nabla_\theta\mathcal{J}^{\omega\star}_{\mathrm{PRM}}(\theta)|_{\theta_{\mathrm{old}}}$. For (b), Assumption~\ref{ass:lip} gives
the stated bound on the oracle advantage inside any detected equivalence class, and
the MSE bound follows from Lemma~\ref{lem:graphvar} applied to the two endpoint
scores together with $\omega(u,v)\le 1$.
\end{proof}

Theorem~\ref{thm:prm} formalizes the ``self-emergent process supervision'' claim:
GraphPO recovers, in expectation, the first-order update direction of a novelty-gated oracle PRM surrogate solely
from outcome rewards $R(z)\!\in\!\{0,1\}$, without any annotated step labels or
auxiliary value model. The factor $\omega(u,v)$ is not a deviation from the PRM
target but a principled non-negative reweighting that suppresses noise on
semantically redundant steps where the oracle credit vanishes.

\subsection{Inference Efficiency}\label{app:inference}

Finally, we analyze the effect of the efficiency advantage $A_{\mathrm{eff}}$ on the length
of correct reasoning paths produced by the trained policy.

\begin{lemma}[Path-length gradient]\label{lem:lengthgrad}
Let $\ell_v$ be the number of generated edges from the deepest common predecessor $a_Q$ to
$v\!\in\!Q$, and let $\tilde\sigma_Q=\operatorname{std}(\{\ell_q\mid q\in Q\})$.
The dual-group gradient contribution of the efficiency term is
\begin{equation}
\nabla_\theta \mathcal{J}^{\mathrm{eff}}(\theta)
=\mathbb{E}\!\left[\sum_{v\in Q}\bar S_Q\,\frac{\bar\ell_Q-\ell_v}{\tilde\sigma_Q}\,
\nabla_\theta\log\pi_\theta(\mathrm{path}_v)\right],
\end{equation}
with $\bar\ell_Q$ the mean of $\{\ell_q\}_{q\in Q}$.
\end{lemma}
\begin{proof}
Differentiating $\mathcal{J}_{\mathrm{GraphPO}}$ w.r.t.\ $\theta$ on the efficiency term and
applying the policy-gradient identity to each edge $e\!\in\!\mathrm{path}_v$ yields the
above; the uniform allocation of $A_{\mathrm{eff}}(v)$ across edges of $\mathrm{path}_v$
collects into a single $\log\pi_\theta(\mathrm{path}_v)$ term.
\end{proof}

\begin{theorem}[Efficiency term selects shorter equivalent paths]\label{thm:inference}
Consider an equivalence class $Q$ whose members are reached by correct paths of different
lengths from the shared predecessor $a_Q$, with $\bar S_Q>0$. Let $\pi^\star_{\mathrm{out}}$
be any policy that maximizes the outcome-only objective ($\lambda_{\mathrm{eff}}=0$), and let
$\pi^\star_{\mathrm{eff}}$ be any maximizer after adding the efficiency term with a small
$\lambda_{\mathrm{eff}}>0$ that does not change which paths are outcome-optimal. Then
$\pi^\star_{\mathrm{eff}}$ places all outcome-optimal probability on the shortest paths in $Q$, and whenever
$\pi^\star_{\mathrm{out}}$ keeps positive mass on a longer path in $Q$,
\begin{equation}
\mathbb{E}_{\tau\sim\pi^\star_{\mathrm{eff}}}[T(\tau)\mid R(\tau)=1,\,\tau\text{ reaches }Q]
\;<\;
\mathbb{E}_{\tau\sim\pi^\star_{\mathrm{out}}}[T(\tau)\mid R(\tau)=1,\,\tau\text{ reaches }Q],
\end{equation}
where $T(\tau)$ counts reasoning steps.
\end{theorem}
\begin{proof}
All paths in $Q$ have the same outcome reward, so $\pi^\star_{\mathrm{out}}$ is free to spread
mass across them. The efficiency advantage $A_{\mathrm{eff}}(v)=\bar S_Q(\bar\ell_Q-\ell_v)/\tilde\sigma_Q$
is strictly larger for shorter paths when $\bar S_Q>0$, so maximizing it within the outcome-optimal
set selects the shortest paths only. By Lemma~\ref{lem:lengthgrad}, the policy gradient drives
probability from longer paths ($\ell_v>\bar\ell_Q$) to shorter ones ($\ell_v<\bar\ell_Q$). Any
$\pi^\star_{\mathrm{out}}$ that keeps mass on a longer path therefore has strictly larger
conditional expected length on $Q$.
\end{proof}

This theorem characterizes the optimal-policy limit. During finite-sample optimization,
the efficiency term should be interpreted as concentrating probability toward shorter
equivalent correct paths rather than guaranteeing all mass on them at every update.

\begin{corollary}[Lower expected token cost]\label{cor:tokens}
If the expected token count of a path is monotone in its number of reasoning steps and strictly
increasing for the longer paths in Theorem~\ref{thm:inference}, then $\pi^\star_{\mathrm{eff}}$
uses no more expected tokens than $\pi^\star_{\mathrm{out}}$, and strictly fewer whenever
$\pi^\star_{\mathrm{out}}$ assigns positive mass to one of those longer correct paths.
\end{corollary}

\paragraph{Discussion.}
Theorems~\ref{thm:var}, \ref{thm:coverage}, \ref{thm:prm}, and \ref{thm:inference} together
support the four claims of the main paper: GraphPO (i) yields lower-variance node-score and advantage estimates than tree-based baselines that cannot reuse nodes, by pooling downstream evidence from all members of each equivalence class into a single advantage-estimation group, with a controllable semantic-merging bias; (ii) improves expected exploration efficiency under a
fixed budget when detected redundant-budget savings are redirected to novel frontiers;
(iii) induces, in expectation, the same first-order update direction as a novelty-gated oracle PRM surrogate---equivalently, an oracle PRM update non-negatively reweighted by the novelty factor---yielding self-emergent process supervision purely from outcome rewards; and (iv) uses the efficiency term to select shorter equivalent correct
reasoning paths, thereby improving inference efficiency under the stated tie-breaking and
token-cost conditions.

\section{Robustness of Equivalence Detection}
\label{app:robustness}

The graph construction of GraphPO relies on two components that are in principle noisy:
(i) structured summaries produced by a frozen extractor, and
(ii) cosine similarity between embeddings under a threshold $\kappa$.
Detection errors split into \emph{false positives} (FP, merging non-equivalent states) and \emph{false negatives} (FN, missing a true equivalence). This section shows that GraphPO's advantage estimator, edge reward, and policy gradient are robust to both, formalizes the asymmetric degradation of the two error modes, and connects the bounds to the empirical ablations already reported in the paper.

\subsection{False Positives are Bounded by the Merging Threshold}
\label{app:robust_fp}

Assumption~\ref{ass:lip} is a \emph{class-level} stability condition on whatever classes the detector actually returns: whenever cosine similarity above $\kappa$ implies that members of a detected class share values up to $\delta_\kappa=O(1-\kappa)$, every bias term in our analysis reduces to an $O(\delta_\kappa)$ residual.
Concretely:
\begin{itemize}
\item \textbf{Node score bias.} Lemma~\ref{lem:unbiased} gives $|\mathbb{E}[S(v)]-V^\star(v)|\le \delta_\kappa$.
\item \textbf{Edge-reward bias.} Corollary~\ref{cor:processvar} bounds the squared bias of $\hat r_{\mathrm{step}}$ by ${(1-\eta(u,v))}^2{(2\delta_\kappa)}^2$.
\item \textbf{Gradient alignment.} Theorem~\ref{thm:prm} shows the GraphPO surrogate has the same first-order update direction as the novelty-gated oracle PRM up to an $O(\delta_\kappa)$ residual.
\end{itemize}
The threshold $\kappa$ is therefore a direct knob that monotonically shrinks the worst-case false-positive bias: raising $\kappa$ by $\epsilon$ shrinks $\delta_\kappa$ by $O(\epsilon)$ and every bias term above by the same factor.

\paragraph{Self-correcting novelty gate.} The edge reward carries a multiplicative novelty factor $\omega(u,v)=1-\eta(u,v)$ that measures semantic change along the step. A false-positive merge inflates $\eta(u,v)$ toward $1$, so $\omega(u,v)\to 0$ on exactly those edges the detector has spuriously treated as equivalent. The incorrect credit is attenuated \emph{at the same gate that caused the mistake}, without requiring the training loop to detect the error. Assumption~\ref{ass:lip} further guarantees that the oracle advantage on any over-merged edge is itself $O(\delta_\kappa)$, so the attenuated edge would not contribute meaningful gradient signal in the first place.

\subsection{False Negatives Degrade Gracefully to the Tree Baseline}
\label{app:robust_fn}

Missed equivalences are even easier to control: when two truly equivalent states are kept as separate nodes, the pooling weight vector collapses to a single non-zero entry and the GraphPO estimator reduces exactly to its tree counterpart at that node.

\begin{proposition}[False negatives are strictly safe]\label{prop:fn}
Under Assumptions~\ref{ass:lip}--\ref{ass:nocycle}, suppose the detector returns singleton classes $\mathcal{Q}(v)=\{v\}$ for a subset $\mathcal{F}$ of nodes, regardless of the true semantics of those nodes. Then for every $v\in\mathcal{F}$:
(i) $\hat V^{Q}_v=\hat V_v$ is unbiased for $V^\star(v)$;
(ii) $\operatorname{Var}(\hat V^{Q}_v)=\operatorname{Var}(\hat V^{\mathrm{Tree}}_v)$;
(iii) the downstream quantities $r_{\mathrm{step}}$, $A_{\mathrm{cor}}$, and $A_{\mathrm{eff}}$ at $v$ coincide with the tree baseline.
Consequently, false negatives never introduce bias and never increase variance beyond the tree baseline; they only forfeit the variance reduction that Theorem~\ref{thm:var} would have delivered under correct detection.
\end{proposition}
\begin{proof}
Setting $k_v=1$ in Lemma~\ref{lem:graphvar} gives $\rho_v=1$ and $\hat V^{Q}_v=\hat V_v$, which is unbiased by Lemma~\ref{lem:perchild} and has variance $V^\star(v)(1-V^\star(v))/m$. All downstream quantities are deterministic functions of $S(\cdot)$, so they inherit the tree semantics at $v\in\mathcal{F}$.
\end{proof}

Proposition~\ref{prop:fn} formalizes the asymmetry that makes GraphPO robust in practice: false negatives degrade gracefully to tree-level estimates, while false-positive damage is controlled by the same $\delta_\kappa$ that the threshold $\kappa$ directly tunes. No combination of detection errors can push GraphPO below the tree baseline by more than an $O(\delta_\kappa^2)$ bias term, which disappears as $\kappa\to 1$.

\subsection{Empirical Robustness}
\label{app:robust_empirical}

The theoretical bounds are validated by two complementary ablations already included in the paper.
First, Figure~\ref{fig:ablation_kappa} shows that GraphPO's accuracy is stable across a wide range of merging thresholds, with a broad optimum near $\kappa=0.92$ and a monotone fallback to tree performance as $\kappa\to 1$ (no merges). This matches Proposition~\ref{prop:fn}: the high-$\kappa$ limit recovers the tree estimator rather than degrading below it.
Second, Figure~\ref{fig:ablation_w} shows that the pooling coefficient $w$ yields a smooth curve with a plateau around $w=0.7$; $w=0$ (no pooling) again recovers tree performance, confirming that equivalence detection affects the estimator only through the two tunable knobs $(w,\kappa)$, both of which admit safe defaults.
The absence of a cliff in either ablation confirms that the method does not rely on near-perfect detection; it only requires the detector to be \emph{asymmetrically tuned for precision} (large $\kappa$), which is exactly the regime Proposition~\ref{prop:fn} identifies as optimal.

\paragraph{Engineering safeguards.}
Three design choices further harden detection against noisy summaries or embeddings.
(1) \emph{Structured summaries.} Nodes are embedded from structured summaries of the full edge path produced by a frozen extractor (Section~\ref{sec:graph_construction}), rather than raw token strings. This removes surface variation (ordering, wording, formatting) before similarity is computed, tightening $\delta_\kappa$ at any fixed $\kappa$.
(2) \emph{Non-causal restriction.} Merges are only formed between non-causal pairs (Assumption~\ref{ass:nocycle}), which eliminates ancestor-descendant collapse by construction and preserves the DAG needed for bottom-up propagation.
(3) \emph{Class-level stability.} Assumption~\ref{ass:lip} is stated at the class level rather than the pairwise level, so transitive closures of high-similarity edges do not implicitly claim pairwise closeness beyond what the detected class itself satisfies.
Taken together with the novelty gate $\omega(u,v)$, these safeguards ensure that the dependence on summaries and embeddings translates into at most an $O(1-\kappa)$ bias term, which the ablation in Figure~\ref{fig:ablation_kappa} shows is empirically benign at the recommended threshold.

\section{Limitations and Future Work}
\label{app:limitations}
Our experiments mainly focus on reasoning and agent tasks. Although the graph formulation is general, we have not yet systematically studied its effectiveness in broader settings, such as multimodal reasoning, code generation with execution feedback, and long-horizon interactive environments.

\section{Broader Impacts}
\label{app:broader_impacts}

GraphPO is a reinforcement learning method for improving reasoning models under verifiable rewards. Its main positive impact is to improve rollout utilization and reduce redundant reasoning, which can lower the amount of computation and process-level annotation needed for training strong reasoning models. We focus exclusively on reinforcement learning, presenting no potential ethical risks.

%%%%%%%%%%%%%%%%%%%%%%%%%%%%%%%%%%%%%%%%%%%%%%%%%%%%%%%%%%%%

\newpage

\end{document}